\documentclass[11pt]{article}
\usepackage{fullpage}
\usepackage[numbers, compress]{natbib}

\usepackage{microtype}
\usepackage{graphicx}
\usepackage{subfigure}
\usepackage{booktabs} 
\usepackage[table]{xcolor}  
\definecolor{Gray}{gray}{0.9}
\definecolor{midgreen}{rgb}{0.1,0.5,0.1}
\definecolor{darkgray}{gray}{0.25}
\definecolor{lightblue}{rgb}{0.25,0.25,0.8}
\definecolor{mydarkblue}{rgb}{0,0.08,0.45}
\usepackage[colorlinks,
linkcolor=lightblue,
anchorcolor=blue,
urlcolor=mydarkblue,
citecolor=lightblue]{hyperref}

\usepackage{amsmath}
\usepackage{amssymb}
\usepackage{mathtools}
\usepackage{amsthm}
\usepackage{thm-restate} 

\usepackage[capitalize,noabbrev,nameinlink]{cleveref}

\usepackage{authblk}
\usepackage[noend]{algorithmic}
\usepackage{algorithm}
\usepackage{bbm}

\newtheorem{claim}{Claim}

\newcommand{\EE}{\mathbb{E}}

\newcommand{\sr}{{\tt srank}}
\newcommand{\poly}{\mathrm{poly}}

\newcommand{\norm}[1]{\ensuremath{\left\| #1 \right\|}}
\newcommand{\normop}[1]{\ensuremath{\left\| #1 \right\|_{\mathrm{op}}}}

\newcommand{\wt}{\widetilde}
\newcommand{\op}{\mathrm{op}}
\newcommand{\att}{\mathrm{Att}}

\def\diag{{\tt diag}}

\def\0{{\bm 0}}

\def\A{{\rm A}}
\def\B{{\rm B}}

\def\D{{\rm D}}

\def\K{{\rm K}}

\def\Q{{\rm Q}}

\def\U{{\mu}}
\def\V{{\rm V}}

\def\X{{\rm X}}
\def\Y{{\rm Y}}

\def\bPi{{\rm \Pi}}

\def\R{{\rm R}}

\def\Ccal{\mathcal{C}}

\def\Hcal{\mathcal{H}}

\def\Ncal{\mathcal{N}}

\def\Pcal{\mathcal{P}}

\global\long\def\RR{\mathbb{R}}

\usepackage{tikz}
\usetikzlibrary{matrix}
\usetikzlibrary{positioning}

\usetikzlibrary {decorations.fractals,spy}
\usepackage{pgfplots}
\pgfplotsset{compat=newest}
\usepgfplotslibrary{fillbetween}

\usetikzlibrary{decorations.pathmorphing}
\usetikzlibrary{decorations.markings}
\usetikzlibrary{shapes.gates.logic.US,trees,positioning,arrows}

\tikzset{
	arn/.style = {circle, white, draw=black, fill=gray!30, inner sep = 10.5},
	arn_t/.style = {circle, white, draw=black, very thick, fill=gray!30, inner sep = 11.0},
	arn_l/.style = {circle, white, draw=black, very thick, fill=black, inner sep = 2},
	photon/.style={draw=black, very thick, dashed},
	electron/.style={draw=black, very thick},
	tr/.style={buffer gate US,thick,draw,fill=gray!60,rotate=90,	anchor=east,minimum width=2.25cm},
	br/.style={buffer gate US,thick,draw,fill=gray!60,rotate=90,	anchor=east,minimum width=4.5cm},
	brr/.style={buffer gate US,draw,fill=gray!60,rotate=90,	anchor=east,minimum width=4.5cm, opacity = 0.6},
	trr/.style={buffer gate US,thick,draw,fill=gray!60,rotate=90,	anchor=east,minimum width=2.25cm, opacity = 0.6},
	trrr/.style={buffer gate US,draw,fill=white!60,rotate=90,	anchor=east,minimum width=2.25cm, opacity = 0.5},
    black2darr/.style={matrix of nodes, row sep=-\pgflinewidth, column sep=-\pgflinewidth, nodes={draw,black}},
    gray2darr/.style={matrix of nodes, row sep=-\pgflinewidth, column sep=-\pgflinewidth, nodes={draw,lightgray}},
    place/.style={circle,draw=blue,fill=blue}, 
    place_b1/.style={circle,draw=orange,fill=orange},
    place_b2/.style={circle,draw=red!90,fill=red!90}, 
    place_b3/.style={circle,draw=violet!70,fill=violet!70},
    place_b4/.style={circle,draw=blue!70,fill=blue!70}
}

\newcommand{\blackSquare}[2]{\draw[teal, black, thick] (A_sparse-#2.south west) rectangle (A_sparse-#2.north east);}

\newcommand{\blackRectangle}[2]{\draw[teal, black, thick] (A_res-#1.south west) rectangle (A_res-#2.north east);}

\theoremstyle{plain}
\newtheorem{theorem}{Theorem}[section]

\newtheorem{lemma}[theorem]{Lemma}

\theoremstyle{definition}
\newtheorem{defn}[theorem]{Definition}

\theoremstyle{remark}

\usepackage[textsize=tiny]{todonotes}

\title{KDEformer: Accelerating Transformers via Kernel Density Estimation}

\author[1]{Amir Zandieh}
\author[2]{Insu Han$^\dag$}
\author[3]{Majid Daliri$^\dag$}
\author[2]{Amin Karbasi}
\affil[1]{Max-Planck-Institut für Informatik}
\affil[2]{Yale University}
\affil[3]{New York University}

\newcommand\blfootnote[1]{%
  \begingroup
  \renewcommand\thefootnote{}\footnote{#1}%
  \addtocounter{footnote}{-1}%
  \endgroup
}

\begin{document}

\maketitle

\begin{abstract}
Dot-product attention mechanism plays a crucial role in modern deep architectures (e.g., Transformer) for sequence modeling, however, naïve exact computation of this model incurs quadratic time and memory complexities in sequence length, hindering the training of long-sequence models.
Critical bottlenecks are due to the computation of partition functions in the denominator of softmax function as well as the multiplication of the softmax matrix with the matrix of values. 
Our key observation is that the former can be reduced to a variant of the kernel density estimation (KDE) problem, and an efficient KDE solver can be further utilized to accelerate the latter via subsampling-based fast matrix products. 
Our proposed KDEformer can approximate the attention in sub-quadratic time with provable spectral norm bounds, while all prior results merely provide entry-wise error bounds.
Empirically, we verify that KDEformer outperforms other attention approximations in terms of accuracy, memory, and runtime on various pre-trained models. 
On BigGAN image generation, we achieve better generative scores than the exact computation with over $4\times$ speedup. 
For ImageNet classification with T2T-ViT, KDEformer shows over $18\times$ speedup while the accuracy drop is less than $0.5\%$.
\end{abstract}

\section{Introduction}\blfootnote{${}^\dag$Equal contribution.}
Transformers~\cite{vaswani2017attention} have been successfully applied to a wide variety of learning tasks in areas such as natural language processing~\cite{devlin2018bert, yang2019xlnet, brown2020language, raffel2020exploring}, computer vision~\cite{carion2020end,dosovitskiy2021an}, and time series forecasting~\cite{zhou2021informer}.
Although popular, these models face serious scalability limitations because naïve exact computation of their attention layers incurs quadratic (in sequence length) runtime and memory complexities. 
This can inhibit the training of large-scale long-sequence models.

Several algorithms have been proposed to improve Transformers' efficiency via approximating the \emph{softmax matrices} in their attention layers with either sparse matrices~\cite{kitaev2019reformer,daras2020smyrf, roy2021efficient, sun2021sparse} or low-rank matrices~\cite{choromanski2020rethinking,katharopoulos2020transformers}, or a combination of both~\cite{chen2021scatterbrain, zaheer2020big, chen2021pixelated, dass2022vitality}.
However, all prior advances 
solely focused on point-wise approximating the entries of the softmax matrix and fail to provide rigorous approximation guarantees on the final output of the attention mechanism. 
In this work, we design algorithms to approximate the output matrix of attention layers with provable spectral norm guarantees.

\subsection{Problem Formulation and Setting.}

Let $n$ be the number of tokens in the input sequence and $d$ be the dimension of latent representations. 
The \emph{dot-product attention} \cite{vaswani2017attention} is a mapping which takes inputs $\Q, \K, \V \in \RR^{n \times d}$ (interpreted as queries, keys, and values of a dictionary) and outputs the following matrix:
\begin{align*} 
\att(\Q, \K, \V) &:= \D^{-1} \A \V \\ \A := \exp\left( \Q \K^\top / \sqrt{d} \right)&,~~~ \D := \diag(\A \mathbf{1}_n), 
\end{align*}
where $\exp(\cdot)$ is applied in an element-wise manner, $\mathbf{1}_n$ is the ones vector in $\RR^n$, and $\diag(\cdot)$ maps its input vector to a diagonal matrix.
We refer to $\A \in \RR^{n \times n}$ as the \emph{attention matrix} and to $\D^{-1} \A$ as the \emph{softmax matrix}. 
Exact computation of the attention matrix $\A$ takes $\Theta(n^2 d)$ operations and storing it requires $\Theta(n^2)$ memory. Thus, naïve computation of $\mathrm{Att}(\Q, \K, \V)$ requires $\Omega(n^2 d)$ runtime and $\Omega(n^2)$ memory. 
Our aim is to approximate the output matrix $\mathrm{Att}(\Q, \K, \V)$ efficiently while preserving its spectral structure.

Our approach is based on reducing the number of columns of matrix $\A$ using importance sampling. We also devise an efficient estimator for the diagonal scaling matrix $\D$, which bypasses exact and explicit computation of matrix $\A$.
Formally, for any given $\varepsilon >0$ and any $\Q, \K, \V \in \RR^{n \times d}$, we want to quickly find a sampling matrix $\bPi \in \RR^{m \times n}$ with a small number $m = n^{1 - \Omega(1)}$ of rows along with a diagonal matrix $\wt{\D} \in \RR^{n \times n}$, such that the following bound on the \emph{operator norm} of the error is satisfied:
\begin{equation}\label{eq:def-spectral-erro-attention}
\norm{\mathrm{Att}(\Q, \K, \V) - \wt{\D}^{-1} \A \bPi^\top\cdot \bPi \V }_{\op} \le \varepsilon \cdot \norm{\D^{-1}\A}_{\op} \norm{\V}_{\op}.
\end{equation}
Note that $\D^{-1}\A$ is a \emph{row-stochastic (transition) matrix}, so its operator norm is $\norm{\D^{-1}\A}_{\op} \in [1, \sqrt{n}]$.

Given a sampling matrix $\bPi$ with $m$ rows, we can compute the matrix product $\A \bPi^\top \cdot \bPi \V$ in $O(n m d)$ total runtime and $O(n m)$ memory because we only need to compute the $m$ sampled columns of $\A$. Therefore, our main goal is to generate a sampling matrix $\bPi$ with a small number of samples along with a diagonal matrix $\wt{\D}$ which satisfy \cref{eq:def-spectral-erro-attention} using a sub-quadratic runtime in $n$.

All prior approximate attention methods have solely focused on finding an approximate attention matrix $\wt{\A}$ such that $\norm{\A - \wt{\A}}_F$ is small, even though $\A$ is not the ultimate output of attention and the output depends on $\V$ in addition to $\A$. In contrast, we propose the first efficient algorithm for approximating the output matrix $\mathrm{Att}(\Q,\K,\V)$ with spectral bounds as per \cref{eq:def-spectral-erro-attention}  (see~\cref{sec:main_result}).

\subsection{Our Techniques and Results}
We leverage the line of work on efficient \emph{Kernel Density Estimation (KDE)}~\cite{scholkopf2002learning, joshi2011comparing, charikar2017hashing, backurs2018efficient, backurs2019space,siminelakis2019rehashing}.
In the KDE problem, we are given a dataset $\X=\{ x_1, x_2, \ldots x_n \}$ and a kernel function $k(\cdot, \cdot)$ and aim to compute the kernel density $\U_\X(q) = \frac{1}{n} \sum_{i=1}^n k(q , x_i)$ for an arbitrary query point $q$. 
The goal of existing methods in the literature is to estimate this value to $(1+\varepsilon)$ relative error in time $O\left( \varepsilon^{-2} d/\wt{\mu}^{\tau}\right)$ for some $\tau>0$, where $\wt{\mu}$ is a lower bound on $\U_\X(q)$.
Particularly, the best-known algorithm for the Gaussian kernel, due to~\citet{charikar2020kernel}, achieves $\tau = 0.173+o(1)$.

We show that finding the sampling matrix $\bPi$ and diagonal scaling $\wt{\D}$ which satisfy \cref{eq:def-spectral-erro-attention} can be reduced to a generalization of the KDE problem.
First note that the $i^{th}$ diagonal entry of the scaling matrix $\D$ is $\D_{i,i} = \sum_{j=1}^n \exp\left( \frac{\langle q_i, k_j \rangle}{ \sqrt{d}} \right)$, which is indeed the kernel density corresponding to exponential kernel function $k(x,y)=\exp(\langle x, y \rangle)$ and dataset $\frac{1}{d^{1/4}} \cdot \K$ at query point $\frac{1}{d^{1/4}} \cdot q_i$. 
Thus, if we had an efficient KDE procedure for estimating the exponential kernel density up to a multiplicative $(1 \pm \varepsilon)$ factor, we could compute a scaling $\wt{\D}$ that satisfies the spectral guarantee of \cref{eq:def-spectral-erro-attention}.

Additionally, to design an efficient sampling matrix $\bPi$ that satisfies \cref{eq:def-spectral-erro-attention} with small number of rows, the sampling probabilities need to be proportional to the column norms of the softmax matrix $\D^{-1} \A$~\cite{zouzias2013randomized}. 
One can see that the squared norm of the $i^{th}$ column of $\D^{-1} \A$ is $\sum_{j \in [n]} \D_{j,j}^{-2} \exp\left( \frac{2}{\sqrt{d}} \langle q_j, k_i \rangle \right)$, which is a \emph{weighted} exponential kernel density with weights $\left\{ \D_{i,i}^{-2} \right\}_{i \in [n]}$ and dataset $\frac{\sqrt{2}}{d^{1/4}} \cdot \Q$ at query point $\frac{\sqrt{2}}{d^{1/4}} \cdot k_i$. 
Therefore, if we could estimate this weighted exponential kernel density up to some constant multiplicative factor, we could generate a sampling matrix $\bPi$ with small number of samples that satisfies \cref{eq:def-spectral-erro-attention}.

Thus, having a generalized KDE procedure for efficiently evaluating the weighted exponential kernel density, enables us to approximate $\mathrm{Att}(\Q, \K, \V)$ as per \cref{eq:def-spectral-erro-attention}. 
While there is no prior solution for this problem, we show how to translate it to the Gaussian KDE problem, which has witnessed significant recent progress, by applying appropriate transformations on $\K$ and $\Q$ (see \cref{alg-w-exp-kde} and \cref{thm-corrctness-runtime-wexpkde-alg}).

\paragraph{Our Theoretical Results.} We give an algorithm that outputs a diagonal $\wt{\D} \in \RR^{n\times n}$ and a sampling matrix $\bPi \in \RR^{m \times n}$ with $m = O\left( \varepsilon^{-2} \log n \cdot \sr({\D}^{-1} \A) \right)$ samples which satisfy the spectral bound of \cref{eq:def-spectral-erro-attention} with high probability in $n$, where $\sr(\D^{-1}\A)$ denotes the \emph{stable rank} of the softmax matrix.
Our method reduces the memory of attention layers to $mn = {O}\left( \varepsilon^{-2} n \log n \cdot \sr(\D^{-1}\A) \right)$. 
Furthermore, if the Gaussian KDE is supported by an algorithm with runtime $O\left( \varepsilon^{-2} d/\wt{\mu}^{\tau}\right)$ for relative error $1+\varepsilon$, and density lower bound $\wt{\mu}$, then our algorithm's runtime is bounded by $O\left( \varepsilon^{-2} d \cdot n^{1 + \tau} \right)$ for any datasets of queries $\Q$ and keys $\K$ with diameter $\max_{i,j \in [n]} \norm{k_i - q_j}_2^2 = o \left( \sqrt{d} \cdot \log n \right)$, which is strongly sub-quadratic in $n$. 
The current best value for $\tau$ is $\tau = 0.173+o(1)$ due to \cite{charikar2020kernel} and any future progress on Gaussian density evaluation immediately improves our method's runtime.

This result applies to a wide range of practical scenarios where the dimension $d$ is not too large. To see why, note that entries of $\K, \Q$ are typically constant, thus, the diameter is $\max_{i,j \in [n]} \norm{k_i - q_j}_2^2 = O(d)$. Therefore, for any dimension $d = o(\log^2n)$, e.g., $d \approx \frac{\log^2n}{\log \log n}$, our method needs only $O\left(m + \varepsilon^{-2} d \cdot n^{1 + \tau}  \right)$ operations, which is significantly faster than exact computation of $\mathrm{Att}(\Q, \K, \V)$.

\begin{figure}[t]
\vspace{-10pt}
\centering
\includegraphics[width=0.55\textwidth]{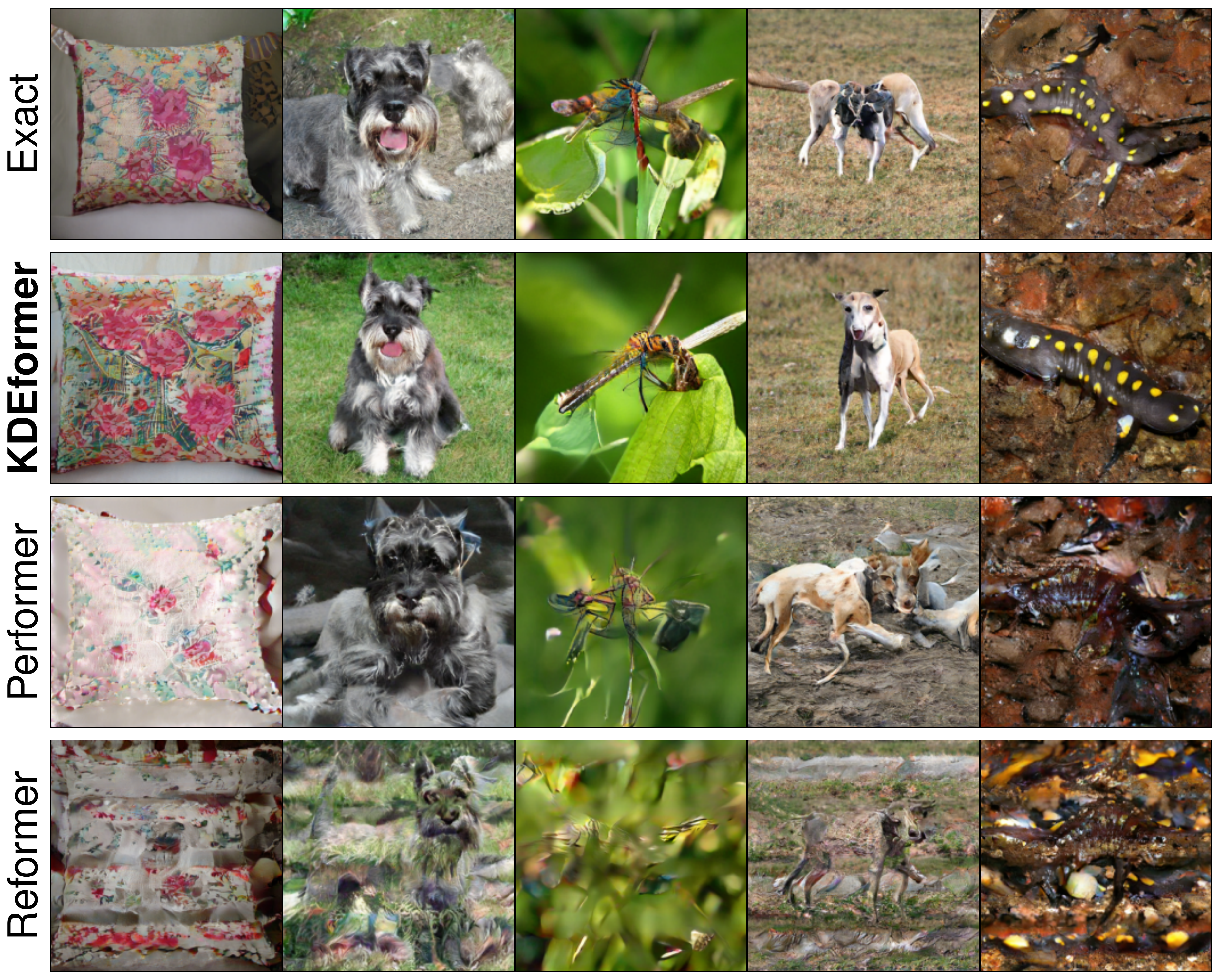}
\vspace{-0.1in}
\caption{Image generations by the pre-trained BigGAN using exact and approximate attention without fine-tuning.}\label{fig:biggan_intro}
\end{figure}

\paragraph{Our Practical Results.}
Our necessary number $m$ of samples depends on the stable rank of the softmax matrix. To reduce $m$, we employ Locality Sensitive Hashing (LSH) to extract the heavy elements of $\D^{-1} \A$ and then show that, in practice, the residual has a significantly smaller stable rank than the original matrix (see~\cref{sec:practical_improvement}).
With this heuristic improvement, we verify that our proposed algorithm outperforms popular attention approximations. In particular, it can save memory space up to $19.06\times$ when the sequence length $n$ is $16{,}394$.
We apply our method to image generation with BigGAN~\cite{brock2018large} and observe that our images, shown in \cref{fig:biggan_intro}, look more natural than others and our generative score is even better than the exact attention. 
Furthermore, for ImageNet classification with Vision Transformer~\cite{yuan2021tokens}, KDEformer shows $18\times$ speedup and $82.08\%$ accuracy which is only $0.5\%$ lower than the exact attention (see~\cref{sec:exp}).
Finally, we demonstrate our method on end-to-end training under the Long Range Arena benchmark~\cite{tay2020long} and observe up to 8$\times$ speedup on wall-clock time than the exact attention (see~\cref{sec:end}).

\subsection{Prior Work}
Several popular methods try to approximate the \emph{heavy} entries of the attention matrix $\A$ by restricting the attention to local neighbors of queries using Locality Sensitive Hashing (LSH)~\cite{kitaev2019reformer, chen2020mongoose, sun2021sparse} or $k$-means clustering~\cite{daras2020smyrf, roy2021efficient}.
Such approaches, however, only provide error bounds on the attention matrix, e.g., guarantees of the form $\| \A - \wt{\A}\|_F < \varepsilon n$, and cannot provide any provable guarantees for the final output matrix $\mathrm{Att}(\Q,\K,\V)$. 
Remarkably, at the core of our algorithm, there are invocations of the Gaussian KDE primitive from \citet{charikar2020kernel}, which heavily employs LSH to estimate kernel densities. 
In contrast to previous works, our algorithm uses LSH in a more subtle way, that is for estimating the right sampling probabilities in order to generate $\bPi$ and also to approximate the scaling $\D$. This difference of approach allows us to approximate $\mathrm{Att}(\Q,\K,\V)$ with spectral norm guarantees.

Another recent line of work is based on approximating the attention matrix $\A$ via random feature maps of the Gaussian or exponential kernels~\cite{choromanski2020rethinking,katharopoulos2020transformers}. 
\citet{chen2021scatterbrain} has recently shown that using a combination of both LSH-based and random features based methods works better at approximating the attention matrix $\A$.
See \cite{tay2022efficient} for a survey.

\section{Preliminaries and Notations}
For any matrix $\A$, we let $a_i$ be its $i^{th}$ row vector and its \emph{stable rank} is defined as $\sr(\A) := \frac{\norm{\A}_F^2}{\norm{\A}_{\op}^2}$ which is always upper bounded by the algebraic rank. 
We denote $e_1, e_2, \ldots e_n$ by the standard basis vectors in $\RR^n$ and $\mathbf{1}_n$ and $\mathbf{0}_n$ by the all-ones and all-zeros vectors in $\RR^n$. For vectors $x,y$ their \emph{direct sum} is denoted by $x\oplus y:= [ x^\top, y^\top]^\top$.

\paragraph{Gaussian KDE.} Our main algorithm is tightly related to the Gaussian KDE, where one is given a dataset $\X \in \RR^{n \times d}$ and wants to build a data-structure (DS) such that given this DS one can estimate the following kernel density value up to $(1 + \varepsilon)$ relative error for any query point $q \in \RR^d$:
\begin{equation}\label{eq:def-Gauss-kde}
\U_{\X}(q) := \frac{1}{n} \sum_{i\in [n]} \exp(-\norm{q - x_i}_2^2 / 2). \end{equation}
The na\"ive method without any DS requires $\Theta(nd)$ time and memory complexities.
The aim is to minimize the memory needed to store the DS and the query time, ultimately being sublinear in $n$.
The pre-processing time which is needed to construct the DS is also desired to be small. 
There have been significant advances on this problem and the current best result was proposed by \citet{charikar2020kernel} as follows:
\begin{theorem}[Fast Gaussian KDE, Theorem 2 in \cite{charikar2020kernel}]\label{thm-Gauss-kde-moses}
Let $\tau = 0.173+o(1)$. For any dataset $\X \in \RR^{n \times d}$ and any $\varepsilon, \wt{\mu} \in (0,1)$, there exist the following procedures:
\begin{enumerate}
	\item \textsc{PreprocessKDE}$(\X,\varepsilon, \wt{\mu})$ constructs a data-structure named ${\tt DS_{kde}}$ in time $O\left( \varepsilon^{-2} d n / \wt{\mu}^{\tau} \right)$.
	\item Given ${\tt DS_{kde}}$, any query $q \in \RR^d$, and $\U_{\X}(q)$ defined as in \cref{eq:def-Gauss-kde}, \textsc{QueryKDE}$({\tt DS_{kde}}, q)$ approximates the quantity $\U_{\X}(q) \cdot \mathbbm{1}_{\{ \wt{\mu} \le \U_{\X}(q) \}}$ up to $(1 + \varepsilon)$ relative error in ${O}(\varepsilon^{-2} d / \left( \wt{\mu} + \U_{\X}(q) \right)^{\tau})$ runtime.
\end{enumerate}
\end{theorem}

The density lower bound $\wt{\mu}$ required by \cref{thm-Gauss-kde-moses} is unknown to us in advance and we learn this quantity adaptively in \cref{alg-w-exp-kde}. We show in \cref{sec:main_result} that for datasets with bounded diameter $\wt{\mu} = n^{-1-o(1)}$.

\section{Efficient Attention with Spectral Bounds}\label{sec:algorithm}
In this section, we design KDEformer which can efficiently compute a sampling matrix $\bPi$ and a diagonal scaling $\wt{\D}$ satisfying \cref{eq:def-spectral-erro-attention}.
We start by showing that this can be done very efficiently given access to a primitive for estimating the row-norms of the attention matrix $\A$ as well as the column-norms of the softmax matrix $\D^{-1}\A$. Next, in \cref{sec:expKDE_construction}, we present a reduction from norm estimators for $\A$ and $\D^{-1}\A$ to the Gaussian KDE problem which has an efficient solution.
Finally, we prove our main result in \cref{sec:main_result}

\subsection{High-level Architecture of the Algorithm}\label{sec:high-level-architecture}
Here, we assume that we have access to an oracle, which can estimate the \emph{weighted} linear combination of $n$ exponential kernels at arbitrary query points, and given this oracle, we design an algorithm that can output $\bPi$ and $\wt{\D}$ which satisfy \cref{eq:def-spectral-erro-attention}. 
In other words, we translate and reduce the problem of spectrally approximating $\mathrm{Att}(\Q, \K, \V)$ to a weighted KDE problem corresponding to the exponential dot-product kernel.
The precise interface and desired properties of this oracle are presented in the following definition,

\begin{defn}[Weighted Exponential KDE]\label{def:ExpKDE}
Let $\X , \Y \in \RR^{n \times d}$ be arbitrary datasets and let $v \in \RR_+^n$ be an arbitrary vector with positive coordinates. For any $\varepsilon > 0$, primitive \textsc{WExpKDE}$(\X, \Y, v, \varepsilon)$ outputs a non-negative vector $\alpha \in \RR_+^n$ such that:
\begin{equation}\label{eq:wexp-desired-error}
	\alpha_j \in (1 \pm \varepsilon) \cdot \sum_{i \in [n]} v_i \exp(\langle x_i, y_j \rangle) \quad \forall j \in [n].
\end{equation} 
\end{defn}
Now we show how to generate $\bPi$ and $\wt{\D}$ that satisfy \cref{eq:def-spectral-erro-attention}, given access to \textsc{WExpKDE} as per \cref{def:ExpKDE}.

\paragraph{Estimating $\D = \diag \left(\exp\left( \Q \K^\top / \sqrt{d} \right)\mathbf{1}_n \right)$.}
One can easily see that the $j^{th}$ diagonal entry of $\D$ equals:
\begin{align}
\D_{j,j} = \sum_{i \in [n]} \exp\left( {\langle k_i, q_j \rangle}/{\sqrt{d}} \right) \quad \forall j \in [n].
\end{align}
Therefore, if we let $\alpha = \textsc{WExpKDE}\left( \frac{\K}{d^{1/4}}, \frac{\Q}{d^{1/4}} , \mathbf{1}_n, \frac{\varepsilon}{3} \right)$ and define $\wt{\D} = \diag(\alpha)$, then by \cref{def:ExpKDE} and using the fact that entries of $\D$ are positive, we have $(1-\varepsilon/3) \D \preceq \wt{\D} \preceq (1+\varepsilon/3) \D$ 
where $\preceq$ is the Loewner order.
So,
\begin{equation}\label{eq:error-D-approx}
\norm{ \mathrm{Att}(\Q, \K, \V) - \wt{\D}^{-1} \A \V }_{\op} \le \frac{\varepsilon}{2} \cdot \norm{ \D^{-1}\A \V}_{\op}.
\end{equation}
Hence, we can estimate $\D$ to sufficient precision by invoking $\textsc{WExpKDE}\left( \frac{\K}{d^{1/4}}, \frac{\Q}{d^{1/4}} , \mathbf{1}_n, \frac{\varepsilon}{3} \right)$.

\paragraph{Generating the Sampling Matrix $\bPi$.}
Given a diagonal matrix $\wt{\D}$ which satisfies \cref{eq:error-D-approx}, by triangle inequality, in order to satisfy the spectral bound of \cref{eq:def-spectral-erro-attention}, it suffices to find a sampling matrix for which the following holds,
\begin{equation} \label{eq:error-bound-sampler-attn}
\norm{  \wt{\D}^{-1} \A \bPi^\top \cdot \bPi \V - \wt{\D}^{-1} \A \V }_{\op} \le \frac{\varepsilon}{2} \cdot \norm{ \D^{-1}\A}_{\op} \norm{ \V}_{\op} \end{equation}
So, our goal is to design a sampling matrix $\bPi \in \RR^{m \times n}$ with a small number $m$ of rows that satisfies \cref{eq:error-bound-sampler-attn}. This problem is in fact well studied in the randomized numerical linear algebra literature and is known as the \emph{Approximate Matrix Multiplication} (AMM) with respect to the
spectral norm. It is known how to achieve the above guarantee using a sampling matrix with $m = O\left( \varepsilon^{-2} \log n \cdot (\sr(\D^{-1} \A) + \sr(\V)) \right)$ i.i.d. rows.

More formally, we have the following result which is a slight modification of Theorem 2.1 from \cite{zouzias2013randomized} and is proved in \cref{appndx-amm-proof}.

\begin{restatable}[AMM]{lemma}{lemammoperatornorm}\label{lem:amm-operator-norm}
For any matrices $\X \in \RR^{n \times q}, \Y \in \RR^{n \times d}$ and any probability distribution $\{ p_i \}_{i \in [n]}$ satisfying $p_i \ge \frac{1}{4} \cdot \frac{\norm{x_i}_2^2 + \gamma \cdot \norm{y_i}_2^2}{\norm{\X}_F^2 + \gamma \cdot \norm{\Y}_F^2}$ for all $i \in [n]$ and $\gamma = \norm{\X}_\op^2 / \norm{\Y}_\op^2$, a sampling matrix $\bPi \in \RR^{m \times n}$ constructed by first generating $m$ i.i.d. samples $\ell_1, \ldots \ell_m \in [n]$ according to $\{p_\ell\}_{\ell\in[n]}$ and then letting the $r^{th}$ row of $\bPi$ be $\frac{1}{\sqrt{m\cdot p_{\ell_r}}} \cdot e_{\ell_r}^\top$, if $m = \Omega\left( \varepsilon^{-2} \log n \cdot (\sr(\X) + \sr(\Y)) \right)$ for some $\varepsilon>0$, the following holds,
\begin{align*}
	\Pr \left[ \normop{\X^\top \bPi^\top \bPi \Y - \X^\top\Y} > \varepsilon\normop{\X} \normop{\Y}\right] \le \frac{1}{\poly(n)}.
\end{align*}
\end{restatable}

So, by invoking \cref{lem:amm-operator-norm} with $\X^\top = \wt{\D}^{-1} \A$ and $\Y = \V$ and error parameter $\varepsilon/2$, we can find a random sampling matrix $\bPi$ which satisfies \cref{eq:error-bound-sampler-attn} with high probability in $n$, as long as the number of samples is at least $m = \Omega\left( \varepsilon^{-2} \log n (\sr(\wt{\D}^{-1} \A) + \sr(\V)) \right)$. 
The only catch is that, to apply \cref{lem:amm-operator-norm}, we need to compute the distribution $\{ p_i \}_{i \in [n]}$ as per this lemma. 
In other words, we need to compute the row norms of $\V$ as well as the column norms of $\wt{\D}^{-1} \A$. All row norms of $\V$ can be computed in $O(nd)$ time. However, naively computing the column norms of $\wt{\D}^{-1} \A$ would require $\Theta(n^2d)$ operations.
Fortunately, the column norms of $\wt{\D}^{-1} \A$ can be approximated via the primitive \textsc{WExpKDE} from \cref{def:ExpKDE}.

\begin{algorithm}[t]
\caption{KDEformer} \label{alg-outer-loop}
\begin{algorithmic}[1]
	\STATE {\bf input}: matrices $\Q, \K, \V \in \RR^{n \times d}$, integer $m$, and $\varepsilon>0$
	\STATE $\gamma \gets \norm{\V}_\op^{-2}$ via power method  \label{line-gamma}
	\STATE $\alpha \gets \textsc{WExpKDE}\left( \frac{\K}{d^{1/4}} , \frac{\Q}{d^{1/4}} , \mathbf{1}_n, \frac{\varepsilon}{3} \right)$ in \cref{def:ExpKDE} \label{line-row-norm-attn-outerlopp}
	\STATE $\beta \gets \textsc{WExpKDE}\left( \frac{\sqrt{2} \cdot \Q}{d^{1/4}}, \frac{\sqrt{2} \cdot \K}{d^{1/4}}, u, 1 /3 \right)$, 
	where $u_i \gets 1/\alpha_i^2$ for every $i \in [n]$ \label{line-col-norm-attention-matrix}
	\STATE $p_i \gets \beta_i + \gamma \cdot \norm{v_i}_2^2$ for every $i \in [n]$ then 
	normalize $p_\ell \gets \frac{p_\ell}{\sum_{j\in [n]} p_j}$ for every $\ell \in [n]$ \label{line-dist-p}
	\STATE generate i.i.d. samples $\ell_1, \ell_2, \ldots \ell_m \in [n]$ from distribution $\{p_\ell\}_{\ell\in[n]}$\label{line-gen-samples-alg}
	\STATE let $r^{th}$ row of $\bPi$ be $\frac{1}{\sqrt{m\cdot p_{\ell_r}}} \cdot e_{\ell_r}^\top$ for every $r \in [m]$ \label{line-sample-matrix-construct}
	\STATE {\bf return} $\wt{\D} = \diag(\alpha)$ and $\bPi$
\end{algorithmic}
\end{algorithm}

The procedure for computing  $\wt{\D}$ and sampler $\bPi$ is presented in \cref{alg-outer-loop}.
We state the correctness of \cref{alg-outer-loop} in the following theorem and prove it in \cref{sec:proof-correctness-outer-loop}.

\begin{restatable}[Correctness of \cref{alg-outer-loop}]{theorem}{thrmouterloop}\label{thm-correctness-outerloop}
For any matrices $\Q, \K, \V \in \RR^{n \times d}$, any $\varepsilon>0$, and number of samples $m = \Omega\left( \varepsilon^{-2} \log n \cdot (\sr({\D}^{-1} \A) + \sr(\V)) \right)$, given access to a primitive \textsc{WExpKDE} as per \cref{def:ExpKDE}, \cref{alg-outer-loop} outputs a diagonal matrix $\wt{\D} \in \RR^{n\times n}$ and a sampling matrix $\bPi \in \RR^{m \times n}$ which satisfy \cref{eq:def-spectral-erro-attention} with probability at least $1 - \frac{1}{\poly(n)}$.
\end{restatable}

So, to spectrally approximate $\mathrm{Att}(\Q, \K, \V)$, it is enough to run \cref{alg-outer-loop}. 
This algorithm relies on the existence of primitive \textsc{WExpKDE} as per \cref{def:ExpKDE}, therefore, we focus on efficient implementation of \textsc{WExpKDE}.

\subsection{Weighted Exponential KDE}\label{sec:expKDE_construction}
Here, we devise an efficient algorithm that satisfies the desired properties of \textsc{WExpKDE} as per \cref{def:ExpKDE}.
We show that this procedure is tightly related to and can be translated to an instance of the Gaussian KDE. 
First note that if all data-points in dataset $\X$ were on a sphere, i.e., $\norm{x_i}_2 = r$ for all $i\in[n]$ and some $r>0$, then the weighted exponential kernel density corresponding to the weights $v = \frac{1}{n} \cdot \mathbf{1}_n$ would be equal to $e^{(\norm{q}_2^2 + r^2)/2}\cdot \U_{\X}(q)$, where $\U_{\X}(q)$ is defined as in \cref{eq:def-Gauss-kde}.

Our proposed \textsc{WExpKDE} primitive employs a fast Gaussian KDE method as per \cref{thm-Gauss-kde-moses}. The weighted exponential kernel density for a query point $q$ and weight vector $v \in \RR_+^{n}$ can be written as,
\begin{align}
\sum_{i \in [n]} v_i e^{\langle x_i, q \rangle} = e^{\frac{ \norm{q}_2^2}{2}} \sum_{i \in [n]} v_i e^{\frac{ \norm{x_i}_2^2}{2}} \cdot e^{-\frac{\norm{x_i - q}_2^2}{2} }. \label{eq:wexpkde-toGauss-kde-reduction-intermed-step}
\end{align}
Let us define $w_i := \sqrt{ 2 \log \frac{\sum_{j \in [n]} v_j \exp(\norm{x_j}_2^2 /2 )}{v_i \cdot \exp( \norm{x_i}_2^2 / 2 )} }$ for every $i \in [n]$ and define the augmented dataset $\X' \in \RR^{n \times (d+1)}$ as $x_i' := x_i \oplus [w_i]$ for every $i \in [n]$. Also let the augmented query point be $q':= q \oplus [0]$. 
Then, the r.h.s. in \cref{eq:wexpkde-toGauss-kde-reduction-intermed-step} can be written as
\begin{align}
e^{\frac{ \norm{q}_2^2}{2}} \sum_{i \in [n]} v_i e^{\frac{ \norm{x_i}_2^2}{2}} \cdot \exp \left(-\frac{ \norm{x_i' - q'}_2^2}{2} + \frac{w_i^2}{2} \right) = n \cdot e^{\frac{\norm{q}_2^2}{2}}  \sum_{j \in [n]} v_j e^{ \frac{\norm{x_j}_2^2}{2} }  \cdot \U_{\X'}(q') \label{eq:expkde-reduced-gauss-kde}.
\end{align}
Therefore, the weighted exponential kernel density can be obtained from the Gaussian kernel density corresponding to the augmented dataset $\X'$ and augmented query $q'$, i.e., $\U_{\X'}(q')$. 
The augmented dataset can be constructed very efficiently in time $O(nd)$, so given a fast Gaussian KDE as per \cref{thm-Gauss-kde-moses}, \cref{eq:expkde-reduced-gauss-kde} shows us an efficient way to implement the \textsc{WExpKDE} procedure. 
Our proposed procedure is presented in \cref{alg-w-exp-kde}. Note that, fast Gaussian KDE requires a lower bound $\wt{\mu}$ on the kernel density value $\U_{\X'}(q')$, and we show how to adaptively learn $\wt{\mu}$ in \cref{alg-w-exp-kde} using the fact that if \textsc{QueryKDE}$({\tt DS_{kde}}, q')$ outputs zero we can infer that our lower bound was too high.
\begin{algorithm}[t]
\caption{Weighted Exponential KDE (\textsc{WExpKDE})} \label{alg-w-exp-kde}
\begin{algorithmic}[1]
	\STATE {\bf input}: matrices $\X, \Y \in \RR^{n \times d}$, vector $v \in \RR_+^n$, error parameter $\varepsilon>0$, and $\tau > 0$
	\STATE $\mu \gets 1/n$ and $S \gets [n]$ and $\alpha \gets \mathbf{0}_n$
	\STATE $N \gets \sum_{j \in [n]} v_j e^{\frac{\norm{x_j}_2^2}{2}}$ 
	\STATE $w_i \gets \sqrt{ 2 \log \frac{N}{v_i \cdot \exp( \norm{x_i}_2^2 / 2 )} }$ for every $i \in [n]$ \label{line-def-N}
	\STATE $\X' \gets [\X ; w] \in \RR^{n\times (d+1)}$, $\Y' \gets [\Y ; \mathbf{0}_n] \in \RR^{n \times (d+1)}$ \label{lin-def-XYprime}
	\WHILE{$ \mu^{-\tau} \le \varepsilon^2 \cdot |S| $}\label{line-while-loop}
	\STATE ${\tt DS_{kde}}\gets\textsc{PreprocessKDE}(\X',\varepsilon, \mu)$ \label{line-Gauss-kde-preprocess}
	\STATE $\alpha_i \gets n \cdot N \cdot e^{\frac{\norm{y_i}_2^2}{2}} \cdot \textsc{QueryKDE}({\tt DS_{kde}}, y_i')$ for every $i \in S$ \label{line-Gauss-kde-query}
	\STATE $\mu \gets \mu/2$ and $S \gets \{ i \in [n] : \alpha_i = 0 \}$ \label{line-set-S}
	\ENDWHILE
	\STATE $\alpha_j \gets \sum_{i \in [n]} v_i \cdot \exp(\langle x_i, y_j \rangle)$ for every $j \in S$ \label{line-final-alpha-update}
	\STATE {\bf return} $\alpha$
\end{algorithmic}
\end{algorithm}
We analyze \cref{alg-w-exp-kde} in the following theorem.
\begin{theorem}[Analysis of \cref{alg-w-exp-kde}]\label{thm-corrctness-runtime-wexpkde-alg}
For every matrices $\X, \Y \in \RR^{n \times d}$, any non-negative vector $v \in \RR_+^n$, and any $\varepsilon \in (0,1)$, and given a fast Gaussian KDE as per \cref{thm-Gauss-kde-moses}, \cref{alg-w-exp-kde} outputs a vector $\alpha \in \RR^n$ which satisfies the desired conditions of \cref{def:ExpKDE} (i.e., \cref{eq:wexp-desired-error}). 
Furthermore, this procedure's runtime is $O\left( n d \cdot \Ccal_{\X, \Y, v, \varepsilon, \tau} \right)$, where 
\begin{align}
	\Ccal_{\X, \Y, v, \varepsilon, \tau} := \min_{\mu > 0} \, \frac{1}{\varepsilon^{2}  \mu^{\tau}} + \left| \left\{ i\in[n]:  \frac{\sum_{j=1}^n v_j e^{\langle x_j, y_i \rangle}}{ \sum_{j=1}^n v_j e^{\frac{\norm{x_j}_2^2 + \norm{y_i}_2^2}{2} } } < n \mu \right\} \right| \label{eq:runtime-bound-WExpKDE}
\end{align}
\end{theorem}

\begin{proof}
First, we prove the correctness. 
Let us index the iterations of the algorithm's while loop by $t=0,1,2, \ldots $ and let $\mu_t$, $\alpha_t$, and $S_t$ denote the value of $\mu$, the vector $\alpha$, and set $S$ at $t^{th}$ iteration. 
We have $|S_t| \le n$ and $\mu_t = \frac{1}{n \cdot 2^t}$ for every $t$, thus, the algorithm must terminate in $T = O(\log n)$ iterations.
Also, by \cref{thm-Gauss-kde-moses}, the set $S_{t+1}$ computed in line~9 equals $S_{t+1} = \{ i\in[n]: \U_{\X'}(y_i') < \mu_t \}$, because the fast Gaussian KDE procedure outputs zero if and only if $\U_{\X'}(y_i') < \mu_t$.

Next, we show by induction that at every iteration $t$, $\alpha_t(i)$ is within $(1\pm \varepsilon)$ factor of $n N  e^{\frac{\norm{y_i}_2^2}{2}} \cdot \U_{\X'}(y_i')$ for all $i \in [n] \setminus S_t$.
{\bf Base of induction} is trivial because $S_0 = [n]$. 
For proving the {\bf inductive step}, note that in lines~7-8 $\alpha_{t+1}(i)$ is updated for every $i \in S_t$ by invoking the fast Gaussian KDE procedure and $\alpha_{t+1}(i) = \alpha_t(i)$ for $i \in [n] \setminus S_t$. 
Thus, by the inductive hypothesis and \cref{thm-Gauss-kde-moses} as well as definition of $S_{t+1}$ in line~9, $\alpha_{t+1}(i)$ is within $(1\pm \varepsilon)$ factor of $n N  e^{\frac{\norm{y_i}_2^2}{2}} \cdot \U_{\X'}(y_i')$ for all $i \in [n] \setminus S_{t+1}$, which completes the inductive proof.
Using the definition of $N$ in line~3 and definition of $\X', \Y'$ in line~5 along with \cref{eq:expkde-reduced-gauss-kde}, the invariant that we proved implies that for every $t = 0,1, \ldots T$, $\alpha_{t}(i)$ is within $(1\pm \varepsilon)$ factor of $\sum_{j \in [n]} v_j \cdot \exp(\langle x_j, y_i \rangle)$ for all $i \in [n] \setminus S_t$. 
After exiting the while loop, $\alpha(i)$ is updated at all $i \in S_{T+1}$ in line~\ref{line-final-alpha-update} as $\alpha(i) = \sum_{j \in [n]} v_j \cdot \exp(\langle x_j, y_i \rangle)$, and $\alpha(i) = \alpha_{T}(i)$ for every $i \in [n] \setminus S_T$.
This proves that the output vector $\alpha$ satisfies \cref{eq:wexp-desired-error}, which completes the correctness proof.

\paragraph{Runtime Analysis.}
The runtime has three components; 

\begin{enumerate}
	\item Time to run \textsc{PreprocessKDE} in line~7. The total time of running this primitive in all iterations $t=0,1, \ldots T$ is $O\left( \sum_{t=0}^T \frac{d \cdot n}{\varepsilon^2} \mu_t^{-\tau} \right)$, by \cref{thm-Gauss-kde-moses}. Since $\mu_t = \frac{1}{n \cdot 2^t}$, this runtime is bounded by $O\left( \frac{d \cdot n}{\varepsilon^2} \mu_T^{-\tau} \right)$.
	
	\item  Time to run \textsc{QueryKDE} in line~8. By \cref{thm-Gauss-kde-moses}, the total time to run this procedure in all iterations is ${O}\left( \frac{d}{\varepsilon^2} \cdot \sum_{t=0}^T \sum_{i \in S_t} \left( \mu_t + \U_{\X'}(y_i') \right)^{-\tau} \right)$. Because $|S_t| \le n$, this runtime complexity is completely dominated by {\bf (1)}.
	
	\item Time to exactly compute the weighted exponential densities of the points with very small $\U_{\X'}(y_i')$ value in line~10. This runtime is bounded by $O(nd \cdot |S_{T+1}|)$.
\end{enumerate}

Now we combine these bounds. Using the assumption that the algorithm terminated at iteration $t=T$, the while loop condition at iteration $T+1$ must fail. 
Therefore, $|S_{T+1}| < \mu_{T+1}^{-\tau} /\varepsilon^2 < 2 \mu_{T}^{-\tau} /\varepsilon^2$.
This shows that the first component of the runtime must dominate the third component. 
So the total time is bounded by $O\left( \frac{d \cdot n}{\varepsilon^2} \mu_T^{-\tau} \right)$. 

Recall that the while loop terminates at iteration $T$ meaning that $ \varepsilon^{-2} \mu_t^{-\tau} \le {|S_t|}$ for every $t = 0,1, \ldots T$ and $ \varepsilon^{-2} \mu_{T+1}^{-\tau} > {|S_{T+1}|}$. 
So, $T$ is the largest integer that satisfies $ \varepsilon^{-2} \mu_T^{-\tau} \le {|S_T|}$.
Also recall that $S_t = \{ i\in[n]: \U_{\X'}(y_i') < \mu_{t-1} \}$ and $\mu_t = \frac{1}{n \cdot 2^t}$. Thus, the runtime of the procedure can be expressed as,
\[ 
O( n d ) \cdot \min_{\mu > 0}  \varepsilon^{-2} \mu^{-\tau} + \left| \left\{ i\in[n]: \U_{\X'}(y_i') < \mu \right\} \right|.
\]
The definition of $\X', \Y'$ in line~5 along with \cref{eq:expkde-reduced-gauss-kde} gives the claimed runtime bound in \cref{eq:runtime-bound-WExpKDE}.
\end{proof}

To get a better understanding of the runtime bound in \cref{thm-corrctness-runtime-wexpkde-alg}, suppose that datasets $\X, \Y$ are such that cardinality of set $\left\{ i\in[n]:  \frac{\sum_{j \in [n]} v_j \exp(\langle x_j, y_i \rangle)}{ \sum_{j \in [n]} v_j \exp\left( \frac{\norm{x_j}_2^2 + \norm{y_i}_2^2}{2} \right)} \le n^{-o(1)} \right\}$
is upper bounded by $O\left( \varepsilon^{-2} \cdot n^{\tau} \right)$.
For such datasets, the runtime of \cref{thm-corrctness-runtime-wexpkde-alg} is bounded by $O\left( {\varepsilon^{-2}}{d} \cdot n^{1 + \tau + o(1)} \right)$, which is strongly sub-quadratic in $n$.

\subsection{Main Result}\label{sec:main_result}
Now we are in a position to prove our main result, i.e., an efficient algorithm that can approximate the attention mechanism with spectral guarantees as per \cref{eq:def-spectral-erro-attention}.

\begin{restatable}[Approximate Attention with Spectral Norm Bound]{theorem}{mainthrmspectralguarantee}\label{thm-main-attenstion-full-alg}
For any matrices $\Q, \K, \V \in \RR^{n \times d}$, any $\varepsilon>0$, and given a fast Gaussian KDE as per \cref{thm-Gauss-kde-moses}, there exists an algorithm that outputs a diagonal matrix $\wt{\D} \in \RR^{n\times n}$ and a sampling matrix $\bPi \in \RR^{m \times n}$ with $m = O\left( \varepsilon^{-2} \log n \cdot (\sr({\D}^{-1} \A) + \sr(\V)) \right)$ samples which satisfy \cref{eq:def-spectral-erro-attention} with probability at least $1 - \frac{1}{\poly(n)}$.
The runtime of this algorithm is $O\left( m + nd \cdot \left( \Ccal_{ \frac{\K}{d^{1/4}} , \frac{\Q}{d^{1/4}}, \mathbf{1}_n, \varepsilon, \tau} +  \Ccal_{\frac{\sqrt{2} \cdot \Q}{d^{1/4}}, \frac{\sqrt{2} \cdot \K}{d^{1/4}}, v, 1, \tau} \right) \right)$,
where $v_{j} = \left( \sum_{\ell \in [n]} \exp\left( \frac{1}{\sqrt{d}} \langle q_j, k_\ell \rangle \right) \right)^{-2}$ for $j \in [n]$ and $\Ccal_{ \frac{\K}{d^{1/4}} , \frac{\Q}{d^{1/4}}, \mathbf{1}_n, \varepsilon, \tau}, \Ccal_{\frac{\sqrt{2} \cdot \Q}{d^{1/4}}, \frac{\sqrt{2} \cdot \K}{d^{1/4}}, v, 1, \tau}$ are defined as in \cref{eq:runtime-bound-WExpKDE}.
\end{restatable}
We prove this theorem in \cref{appendix_proof_main_thrm}.
The runtime bound in \cref{thm-main-attenstion-full-alg} can be simplified for datasets $\Q,\K$ with bounded diameter as follows,

\begin{restatable}[Simplified Runtime for Bounded Diameter Datasets]{corr}{corrsimplifiedruntimr}\label{corr-simplified-runtime}
For any datasets $\Q, \K$ with diameter $\max_{i,j \in [n]} \norm{k_i - q_j}_2^2 = \gamma\sqrt{d}\log n $ for some $\gamma > 0$,
the runtime of \cref{thm-main-attenstion-full-alg} is upper bounded by $O\left(m + nd \cdot \left( n^{\tau(1+\gamma)} + \varepsilon^{-2} n^{\tau(1+\gamma/2)} \right) \right)$, which is strongly sub-quadratic in $n$.
In particular, if $\gamma = o(1)$, the runtime is bounded by $O\left(m + \varepsilon^{-2} d \cdot  n^{1+\tau + o(1)} \right)$.
\end{restatable}

We prove \cref{corr-simplified-runtime} in \cref{appndx-proof-corr}. The current best value for $\tau$ is $\tau = 0.173+o(1)$ due to \citet{charikar2020kernel}, thus, for any datasets of queries $\Q$ and keys $\K$ with diameter $\max_{i,j \in [n]} \norm{k_i - q_j}_2^2 = o(\sqrt{d} \log n)$, our algorithm's runtime is $O\left( m + \varepsilon^{-2} d \cdot n^{1.173+o(1)} \right)$. 

\subsection{Practical Improvements by Exploiting Sparsity}\label{sec:practical_improvement}
Our method relies on a sampling-based AMM (\cref{lem:amm-operator-norm}) and the number of samples $m$ is proportional to $\sr(\D^{-1} \A)$ by \cref{thm-main-attenstion-full-alg}.
Here, 
we propose a practical technique for reducing the stable rank of $\D^{-1} \A$ by finding and subtracting off its ``heavy'' elements. 
Specifically, recall that $\sr(\D^{-1} \A) = \frac{\|{\D^{-1} \A}\|_F^2}{\norm{\D^{-1} \A}_\op^2}$
and the softmax matrix $\D^{-1} \A$ is dominated by its largest elements which correspond to the nearest pairs of queries $q_i$ and keys $k_j$.
Therefore, subtracting off the heavy elements of $\D^{-1} \A$ reduces $\norm{\D^{-1} \A}_F^2$ which in turn can reduce $\sr(\D^{-1} \A)$.

Similar to Reformer~\cite{kitaev2019reformer}, we employ a Locality Sensitive Hashing (LSH) scheme to find dominant entries of the attention matrix $\A$. 
Specifically, let $\Hcal:\RR^d \to [B]$ be an LSH function with $B$ buckets such that the collision probability $\Pr[\Hcal(q_i) = \Hcal(k_j)]$ is ``roughly'' proportional to $\langle q_i , k_j \rangle$.
Given such LSH function, we define the sparse approximation to $\A$ as well as the residual attention matrix as:
\begin{align}
\forall i,j \in [n]:~~~ [\A_{\tt spar}]_{i,j} &:= e^{\frac{\langle q_i , k_j \rangle }{\sqrt{d}} } \cdot \mathbbm{1}_{\{ \Hcal(q_i) = \Hcal(k_j) \}}\nonumber \\ 
\A_{\tt res} &:= \A - \A_{\tt spar}. \label{eq:A_sparse}
\end{align}
Intuitively, the stable rank of $\D^{-1} \A_{\tt res}$ is expected to be smaller than that of $\D^{-1} \A$ because the former has a considerably smaller Frobenius norm. 
We verify this intuition by plotting the singular values distributions of the softmax matrix $\D^{-1} \A$ and the residual $\D^{-1} \A_{\tt res}$ for two real-world instances in \cref{fig:stable_rank}. 
\cref{fig_sing_val_glove} corresponds to when keys and queries are the first $n=2{,}048$ vectors from GloVe word embedding dataset~\cite{pennington2014glove}.
In \cref{fig_sing_val_t2t}, we focused on the first attention layer in Tokens-to-token Vision Transformer (T2T-ViT)~\cite{yuan2021tokens} and an arbitrary batch of images from ImageNet dataset.
In both instances, the singular values of the residual $\D^{-1} \A_{\tt res}$ decay faster than that of $\D^{-1} \A$ while the largest singular value (spectral norm) of both matrices are equal to one.
Thus, as shown in \cref{fig:stable_rank}, subtracting off the sparse component $\D^{-1} \A_{\tt spar}$ reduces the stable rank significantly. 

\begin{figure}[t]
\centering
\subfigure[GloVe dataset]{\label{fig_sing_val_glove}
	\includegraphics[width=0.35\textwidth]{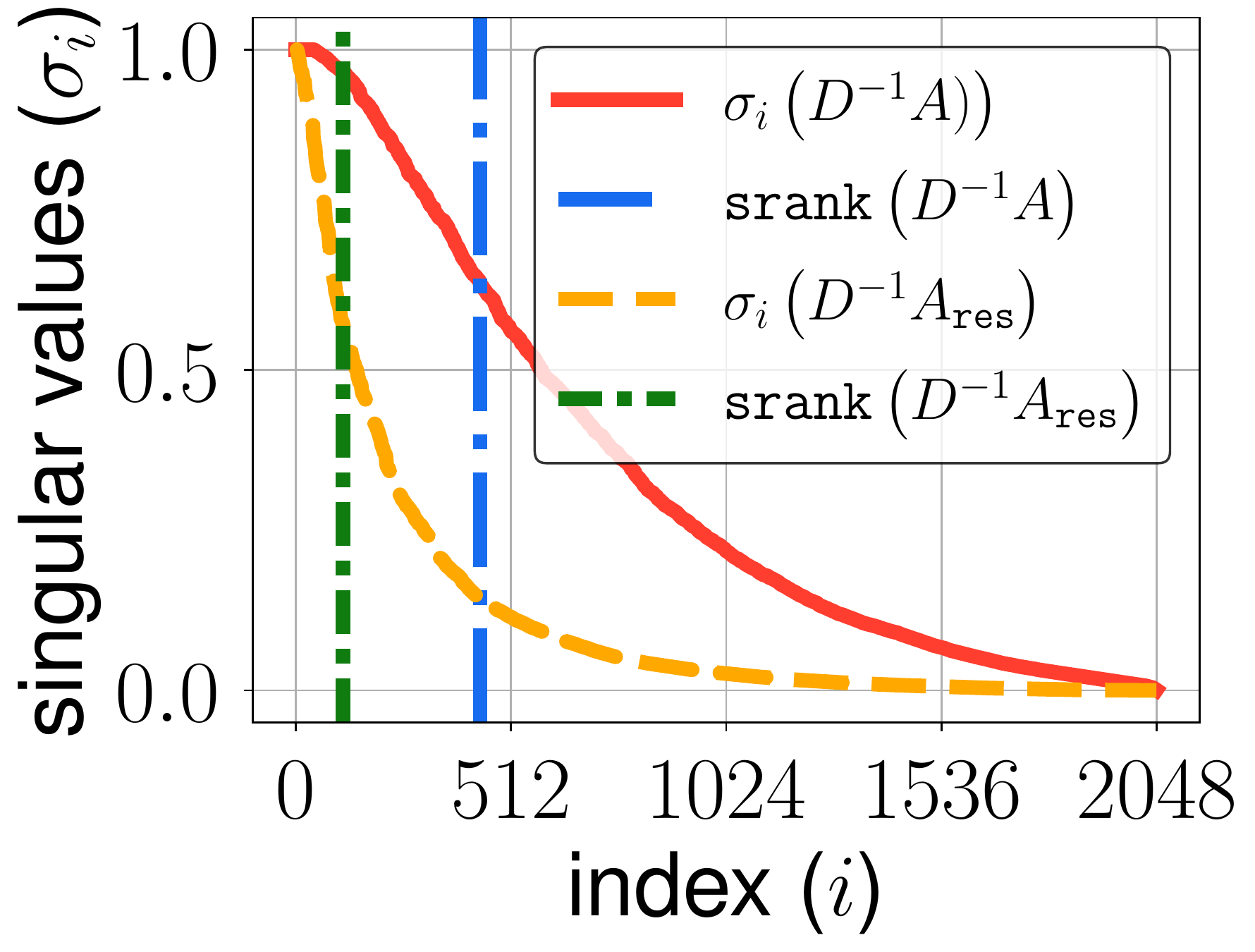}}
 \hspace{0.2in}
\subfigure[T2T-ViT on ImageNet]{\label{fig_sing_val_t2t}
	\includegraphics[width=0.35\textwidth]{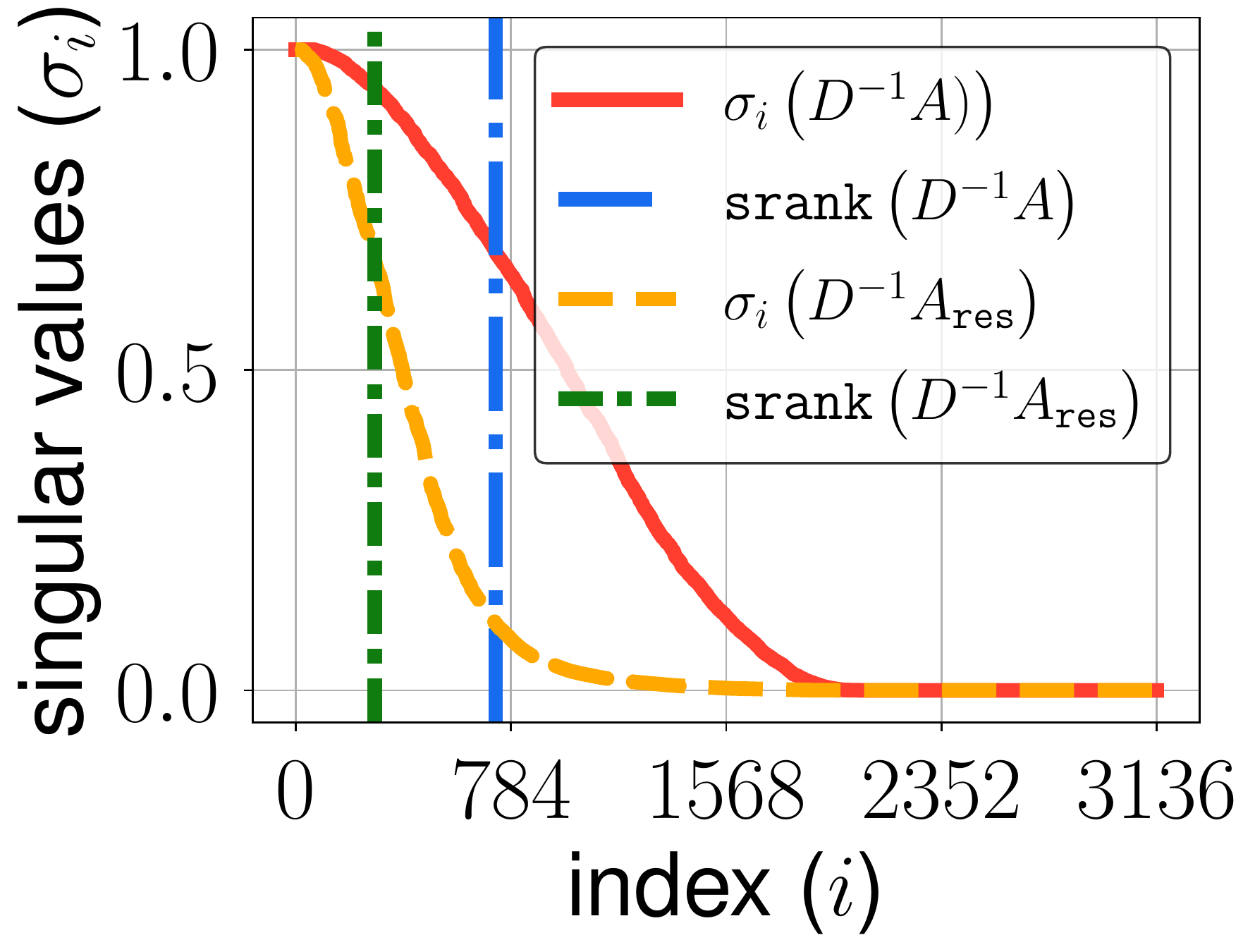}}
\caption{Singular values distribution and stable rank of the softmax matrix $\D^{-1} \A$ versus those of the residual $\D^{-1} \A_{\tt res}$. The stable rank of the residual matrix is significantly smaller.}
\label{fig:stable_rank}
\end{figure}

\begin{algorithm}[t]
\caption{Practical Improvement of KDEformer} \label{alg-practical-adapted}
\begin{algorithmic}[1]
	\STATE {\bf input}: matrices $\Q, \K, \V \in \RR^{n \times d}$, integer $m$, $\varepsilon>0$, and LSH function $\Hcal:\RR^d \to [B]$
	\STATE compute $\alpha, \beta, \gamma$ as per lines~2-4 of \cref{alg-outer-loop}
	\STATE $p_j \gets \beta_j - \sum_{i =1}^n \alpha_{j}^{-2} e^{\frac{2 \langle q_i , k_j \rangle }{\sqrt{d}} } \cdot \mathbbm{1}_{\{ \Hcal(q_i) = \Hcal(k_j) \}} + \gamma  \norm{v_j}_2^2$ for every $j \in [n]$ then normalize $p_\ell \gets \frac{p_\ell}{\sum_{j\in [n]} p_j}$ for every $\ell \in [n]$
	\STATE generate the sampling matrix $\bPi_{\tt res}$ as per lines~6-7 of \cref{alg-outer-loop} using distribution $\{ p_j \}_{j\in[n]}$ computed above	
	\STATE {\bf return} $\wt{\D} = \diag(\alpha)$ and $\bPi_{\tt res}$
\end{algorithmic}
\end{algorithm}

Building upon this observation, we propose a new version of \cref{alg-outer-loop} with improved practical performance.
We start by using \cref{eq:A_sparse} to write:
\begin{equation}\label{eq:attention_sparse_decompose}
\mathrm{Att}(\Q, \K, \V) = \D^{-1} \A_{\tt spar} \V + \D^{-1} \A_{\tt res} \V.
\end{equation}
Given $\D$, the first term above can be computed in time $O(d \cdot {\tt nnz}(\A_{\tt spar}))$, where ${\tt nnz}(\cdot)$ denotes the number of nonzero entries of a matrix. By choosing an appropriate LSH we can ensure that ${\tt nnz}(\A_{\tt spar})$ is almost linear in $n$.

The second term in \cref{eq:attention_sparse_decompose} can be approximated via AMM, similar to what was done in \cref{alg-outer-loop}, however, we need to be able to estimate the column norms of $\D^{-1} \A_{\tt res}$. 
Fortunately, by \cref{eq:A_sparse}, we have $\norm{\D^{-1} \A^j_{\tt res}}_2^2 = \norm{\D^{-1} \A^j}_2^2 - \norm{\D^{-1} \A^j_{\tt sparse}}_2^2$, where $\A^j_{\tt res}, \A^j, \A^j_{\tt sparse}$ denote the $j^{th}$ columns of $\A_{\tt res}, \A, \A_{\tt spar}$, respectively.
Since we can estimate the column norms of $\D^{-1} \A$ efficiently using \textsc{WExpKDE} and all column norms of $\D^{-1} \A_{\tt spar}$ can be computed in total ${\tt nnz}(\A_{\tt spar})$ time, the AMM sampling matrix for residual $\bPi_{\tt res}$ can be generated quickly.

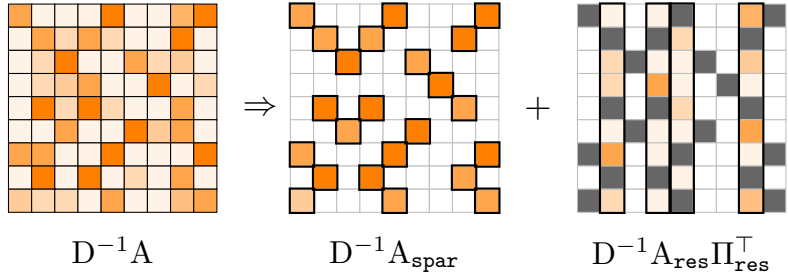
\begin{figure}[t]
\centering
\begin{tikzpicture}[
	every node/.style={scale=1.2}
	]
	\matrix (A) [black2darr]{
		|[fill=orange!70]|\phantom{} & |[fill=orange!10]|\phantom{} &|[fill=orange!10]|\phantom{} & |[fill=orange!10]|\phantom{} & |[fill=orange!100]|\phantom{} & |[fill=orange!10]|\phantom{} & |[fill=orange!10]|\phantom{}& |[fill=orange!60]|\phantom{}& |[fill=orange!100]|\phantom{}\\ 
		|[fill=orange!10]|\phantom{} & |[fill=orange!70]|\phantom{} & |[fill=orange!30]|\phantom{} & |[fill=orange!70]|\phantom{} & |[fill=orange!30]|\phantom{} & |[fill=orange!10]|\phantom{} & |[fill=orange!10]|\phantom{} & |[fill=orange!100]|\phantom{}& |[fill=orange!10]|\phantom{}\\ 
		|[fill=orange!10]|\phantom{} & |[fill=orange!30]|\phantom{} & |[fill=orange!100]|\phantom{} &|[fill=orange!10]|\phantom{} & |[fill=orange!10]|\phantom{} & |[fill=orange!70]|\phantom{} & |[fill=orange!30]|\phantom{} & |[fill=orange!40]|\phantom{}& |[fill=orange!10]|\phantom{}\\ 
		|[fill=orange!10]|\phantom{} & |[fill=orange!30]|\phantom{} & |[fill=orange!40]|\phantom{} & |[fill=orange!70]|\phantom{} & |[fill=orange!10]|\phantom{} & |[fill=orange!10]|\phantom{} & |[fill=orange!100]|\phantom{} & |[fill=orange!10]|\phantom{}& |[fill=orange!30]|\phantom{}\\ 
		|[fill=orange!10]|\phantom{} & |[fill=orange!100]| \phantom{} & |[fill=orange!30]|\phantom{} & |[fill=orange!100]| \phantom{} & |[fill=orange!30]|\phantom{} & |[fill=orange!10]|\phantom{} & |[fill=orange!10]|\phantom{} & |[fill=orange!70]|\phantom{}& |[fill=orange!10]|\phantom{}\\ 
		|[fill=orange!10]|\phantom{} & |[fill=orange!10]|\phantom{} & |[fill=orange!70]|\phantom{} & |[fill=orange!10]|\phantom{} & |[fill=orange!10]|\phantom{} & |[fill=orange!100]|\phantom{} & |[fill=orange!40]|\phantom{}& |[fill=orange!70]|\phantom{}& |[fill=orange!10]|\phantom{}\\ 
		|[fill=orange!70]|\phantom{} & |[fill=orange!70]|\phantom{} & |[fill=orange!10]|\phantom{} & |[fill=orange!10]|\phantom{} & |[fill=orange!100]|\phantom{} & |[fill=orange!10]|\phantom{} & |[fill=orange!10]|\phantom{}& |[fill=orange!10]|\phantom{}& |[fill=orange!100]|\phantom{}\\ 
		|[fill=orange!10]|\phantom{} & |[fill=orange!100]|\phantom{} & |[fill=orange!10]|\phantom{} & |[fill=orange!100]|\phantom{} & |[fill=orange!10]|\phantom{} & |[fill=orange!30]|\phantom{} & |[fill=orange!10]|\phantom{} &  |[fill=orange!70]|\phantom{}& |[fill=orange!10]|\phantom{}\\ 
		|[fill=orange!40]|\phantom{} & |[fill=orange!30]|\phantom{} & |[fill=orange!30]|\phantom{} & |[fill=orange!10]|\phantom{} & |[fill=orange!70]|\phantom{} & |[fill=orange!10]|\phantom{} & |[fill=orange!10]|\phantom{}& |[fill=orange!50]|\phantom{}& |[fill=orange!70]|\phantom{}\\ 
	};
	\node[below=2.8em of A-6-5] {\small $\D^{-1} \A$};
	
	\node[right=0.1em of A] (rightarrow) {$\Rightarrow$};
	
	\matrix (A_sparse) [gray2darr, right=1.8em of A]{
		|[fill=orange!70]|\phantom{} & |[fill=orange!0]|\phantom{} &|[fill=orange!0]|\phantom{} & |[fill=orange!0]|\phantom{} & |[fill=orange!100]|\phantom{} & |[fill=orange!0]|\phantom{} & |[fill=orange!0]|\phantom{}& |[fill=orange!0]|\phantom{}& |[fill=orange!100]|\phantom{}\\
		|[fill=orange!0]|\phantom{} & |[fill=orange!70]|\phantom{} & |[fill=orange!0]|\phantom{} & |[fill=orange!70]|\phantom{} & |[fill=orange!0]|\phantom{} & |[fill=orange!0]|\phantom{} & |[fill=orange!0]|\phantom{} & |[fill=orange!100]|\phantom{}& |[fill=orange!0]|\phantom{}\\
		|[fill=orange!0]|\phantom{} & |[fill=orange!0]|\phantom{} & |[fill=orange!100]|\phantom{} &|[fill=orange!0]|\phantom{} & |[fill=orange!0]|\phantom{} & |[fill=orange!70]|\phantom{} & |[fill=orange!0]|\phantom{} & |[fill=orange!0]|\phantom{}& |[fill=orange!0]|\phantom{}\\
		|[fill=orange!0]|\phantom{} & |[fill=orange!0]|\phantom{} & |[fill=orange!0]|\phantom{} & |[fill=orange!0]|\phantom{} & |[fill=orange!0]|\phantom{} & |[fill=orange!0]|\phantom{} & |[fill=orange!100]|\phantom{} & |[fill=orange!0]|\phantom{}& |[fill=orange!0]|\phantom{}\\
		|[fill=orange!0]|\phantom{} & |[fill=orange!100]| \phantom{} & |[fill=orange!0]|\phantom{} & |[fill=orange!100]| \phantom{} & |[fill=orange!0]|\phantom{} & |[fill=orange!0]|\phantom{} & |[fill=orange!0]|\phantom{} & |[fill=orange!70]|\phantom{}& |[fill=orange!0]|\phantom{}\\
		|[fill=orange!0]|\phantom{} & |[fill=orange!0]|\phantom{} & |[fill=orange!70]|\phantom{} & |[fill=orange!0]|\phantom{} & |[fill=orange!0]|\phantom{} & |[fill=orange!100]|\phantom{} & |[fill=orange!0]|\phantom{}& |[fill=orange!0]|\phantom{}& |[fill=orange!0]|\phantom{}\\
		|[fill=orange!70]|\phantom{} & |[fill=orange!0]|\phantom{} & |[fill=orange!0]|\phantom{} & |[fill=orange!0]|\phantom{} & |[fill=orange!100]|\phantom{} & |[fill=orange!0]|\phantom{} & |[fill=orange!0]|\phantom{}& |[fill=orange!0]|\phantom{}& |[fill=orange!100]|\phantom{}\\
		|[fill=orange!0]|\phantom{} & |[fill=orange!100]|\phantom{} & |[fill=orange!0]|\phantom{} & |[fill=orange!100]|\phantom{} & |[fill=orange!0]|\phantom{} & |[fill=orange!0]|\phantom{} & |[fill=orange!0]|\phantom{} &  |[fill=orange!70]|\phantom{}& |[fill=orange!0]|\phantom{}\\
		|[fill=orange!40]|\phantom{} & |[fill=orange!0]|\phantom{} & |[fill=orange!0]|\phantom{} & |[fill=orange!0]|\phantom{} & |[fill=orange!70]|\phantom{} & |[fill=orange!0]|\phantom{} & |[fill=orange!0]|\phantom{}& |[fill=orange!0]|\phantom{}& |[fill=orange!70]|\phantom{}\\
	};
	\node[below=2.8em of A_sparse-6-5] {\small $\mathbf \D^{-1} \A_{\tt spar}$};

	\node[right=0.1em of A_sparse] (plus) {$+$};
	
	\matrix (A_res) [gray2darr, right=0.1em of plus] {
		|[fill=black!60]|\phantom{} & |[fill=orange!10]|\phantom{} &|[fill=black!0]|\phantom{} & |[fill=orange!10]|\phantom{} & |[fill=black!60]|\phantom{} & |[fill=black!0]|\phantom{} &\phantom{}&  |[fill=orange!60]|\phantom{}& |[fill=black!60]|\phantom{}\\
		|[fill=black!0]|\phantom{} & |[fill=black!60]|\phantom{} & |[fill=black!0]|\phantom{} & |[fill=black!60]|\phantom{} &  |[fill=orange!30]|\phantom{} & |[fill=black!0]|\phantom{} & |[fill=black!0]|\phantom{} & |[fill=black!60]|\phantom{}& |[fill=black!0]|\phantom{}\\
		|[fill=black!0]|\phantom{} &  |[fill=orange!30]|\phantom{} & |[fill=black!60]|\phantom{} & |[fill=orange!10]|\phantom{} &  |[fill=orange!10]|\phantom{} & |[fill=black!60]|\phantom{} & |[fill=black!0]|\phantom{} &  |[fill=orange!40]|\phantom{}& |[fill=black!0]|\phantom{}\\
		|[fill=black!0]|\phantom{} &  |[fill=orange!30]|\phantom{} & |[fill=black!0]|\phantom{} &  |[fill=orange!70]|\phantom{} &  |[fill=orange!10]|\phantom{} & |[fill=black!0]|\phantom{} & |[fill=black!60]|\phantom{} &  |[fill=orange!10]|\phantom{}& |[fill=black!0]|\phantom{}\\
		|[fill=black!0]|\phantom{} & |[fill=black!60]| \phantom{} & |[fill=black!0]|\phantom{} & |[fill=black!60]| \phantom{} &  |[fill=orange!30]|\phantom{} & |[fill=black!0]|\phantom{} & |[fill=black!0]|\phantom{} & |[fill=black!60]|\phantom{}& |[fill=black!0]|\phantom{}\\
		|[fill=black!0]|\phantom{} &  |[fill=orange!10]|\phantom{} & |[fill=black!60]|\phantom{} &  |[fill=orange!10]|\phantom{} &  |[fill=orange!10]|\phantom{} & |[fill=black!60]|\phantom{} & |[fill=black!0]|\phantom{}&  |[fill=orange!70]|\phantom{}& |[fill=black!0]|\phantom{}\\
		|[fill=black!60]|\phantom{} &  |[fill=orange!70]|\phantom{} & |[fill=black!0]|\phantom{} &  |[fill=orange!10]|\phantom{} & |[fill=black!60]|\phantom{} & |[fill=black!0]|\phantom{} & |[fill=black!0]|\phantom{}&  |[fill=orange!10]|\phantom{}& |[fill=black!60]|\phantom{}\\
		|[fill=black!0]|\phantom{} & |[fill=black!60]|\phantom{} & |[fill=black!0]|\phantom{} & |[fill=black!60]|\phantom{} &  |[fill=orange!10]|\phantom{} & |[fill=black!0]|\phantom{} & |[fill=black!0]|\phantom{} &  |[fill=black!60]|\phantom{}& |[fill=black!0]|\phantom{}\\
		|[fill=black!60]|\phantom{} &  |[fill=orange!30]|\phantom{} & |[fill=black!0]|\phantom{} &  |[fill=orange!10]|\phantom{} & |[fill=black!60]|\phantom{} & |[fill=black!0]|\phantom{} & |[fill=black!0]|\phantom{}&  |[fill=orange!50]|\phantom{}& |[fill=black!60]|\phantom{}\\
	};
	\node[below=2.8em of A_res-6-5] {\small $\D^{-1} \A_{\tt res} \Pi_{\tt res}^\top$};
	
	\blackSquare{1-1}{1-1}
	\blackSquare{1-5}{1-5}
	\blackSquare{1-9}{1-9}
	\blackSquare{2-2}{2-2}
	\blackSquare{2-4}{2-4}
	\blackSquare{2-8}{2-8}
	\blackSquare{3-3}{3-3}
	\blackSquare{3-6}{3-6}
	\blackSquare{4-7}{4-7}
	\blackSquare{5-2}{5-2}
	\blackSquare{5-4}{5-4}
	\blackSquare{5-8}{5-8}
	\blackSquare{6-3}{6-3}  
	\blackSquare{6-6}{6-6}
	\blackSquare{7-1}{7-1}
	\blackSquare{7-5}{7-5}  \blackSquare{7-9}{7-9}
	\blackSquare{8-2}{8-2}
	\blackSquare{8-4}{8-4}  \blackSquare{8-8}{8-8}
	\blackSquare{9-1}{9-1}
	\blackSquare{9-5}{9-5}
	\blackSquare{9-9}{9-9}
	
	\blackRectangle{9-8}{1-8}
	\blackRectangle{9-5}{1-5}
	\blackRectangle{9-4}{1-4}  \blackRectangle{9-2}{1-2}
\end{tikzpicture}
\vspace{-5pt}
\caption{The softmax matrix $\D^{-1} \A$ decomposes into its sparse approximation $\D^{-1} \A_{\tt spar}$, which captures large entries (coded with darker colors), and the residual $\D^{-1} \A_{\tt res}$, where black cells represent entries captured by $\D^{-1} \A_{\tt spar}$. Blank colors in $\D^{-1} \A_{\tt res}$ represent columns {\bf not} sampled by AMM sampling matrix $\bPi_{\tt res}$.}
\label{fig:architecture_srank_reduction}
\end{figure}

\begin{figure*}[t]
\centering
\includegraphics[width=0.99\textwidth]{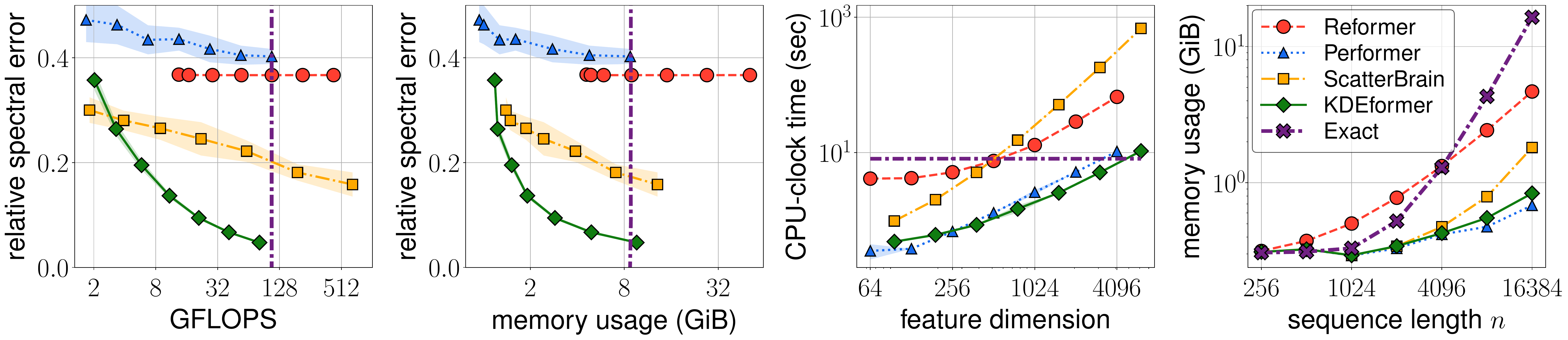}
\vspace{-5pt}
\caption{Performance evaluations of various self-attention approximations on approximating under the GloVe word embeddings.} \label{fig:glove}
\end{figure*}

Putting everything together, we first choose an appropriate LSH function $\Hcal$ and compute the sparse approximation to the attention matrix as per \cref{eq:A_sparse}. 
We show how to design a GPU-friendly LSH whose collision probability $\Pr[\Hcal(q_i) = \Hcal(k_j)]$ is roughly proportional to $\langle q_i , k_j \rangle$ in \cref{appendix_lsh_design}.
Next, we compute a spectral proxy $\wt{\D}$ for $\D$, as was done efficiently in \cref{alg-outer-loop}. 
Finally, we perform AMM on matrices $\wt{\D}^{-1} \A_{\tt res}$ and $\V$ via a sampling matrix $\bPi_{\tt res}$.
The resulting estimator is:
\[
\wt{\mathrm{Att}} = \wt{\D}^{-1} \A_{\tt spar} \V + \wt{\D}^{-1} \A_{\tt res} \bPi_{\tt res}^\top \cdot \bPi_{\tt res} \V.
\]
We illustrate this procedure in \cref{fig:architecture_srank_reduction} and present the pseudocode for computing $\wt{\D}$ and $\bPi_{\tt res}$ in \cref{alg-practical-adapted}.
By an analysis similar to \cref{corr-simplified-runtime}, we find that the runtime of \cref{alg-practical-adapted} is $O ( m + \varepsilon^{-2} d n^{1 + \tau +o(1)} + {\tt nnz}(\A_{\tt spar}) )$ with some $m = O\left( \varepsilon^{-2} \log n \cdot \sr({\D}^{-1} \A_{\tt res})  \right)$.

\section{Experiments} \label{sec:exp}

\subsection{Single Self-attention Layer Approximation} \label{sec:single}
We first benchmark our algorithm on approximating a single self-attention layer, i.e., $\mathrm{Att}(\Q,\K,\V)$. We randomly select a pair of matrices $\Q,\V \in \RR^{n\times d}$ from the GloVe word embeddings~\cite{pennington2014glove} with sequence length $n=8{,}192$ and dimension $d=100$ and set $\K=\Q$. We compare our KDEformer to other attention approximations including Reformer~\cite{kitaev2019reformer}, Performer~\cite{choromanski2020rethinking}, and  ScatterBrain~\cite{chen2021scatterbrain}. 
We compute the relative error under the operator norm, i.e., $\frac{\|\att(\Q,\K,\V) - \wt{\att} \|_{\op}}{\normop{\att(\Q,\K,\V)}}$ where $ \wt{\att} \in\RR^{n \times d}$ is an approximate attention, and measure the peak memory usage, FLOP count and CPU-clock time while varying hyperparameters of algorithms which affect both the runtime and memory space.

In \cref{fig:glove}, we observe that our proposed algorithm achieves the lowest error with minimal FLOP count and memory usage. 
In particular, our approximation error can be about 9\% with 3.06$\times$ memory reduction and 5.11$\times$ lower FLOPS.
In addition, 
we plot CPU-clock time for various choices of hyperparameters that determine peak memory usage. Specifically, if the approximation requires at most $nk$ memory space for computing $\wt{\att}$ and we call $k$ as the feature dimension. Given the same feature dimension, our algorithm and Performer are the fastest methods, but Performer has significantly larger errors than the others. We fix the feature dimension $k=128$ and measure the peak memory usage while the sequence length $n$ is changing from $256$ and $16{,}384$. For $n=16{,}384$, our method can save up to $19.62\times$ memory space compared to the exact computation. 

\subsection{Image generation with BigGAN}
We next apply above-mentioned attention approximations to generate synthetic images with BigGAN~\cite{brock2018large}. The model contains a single attention layer where the corresponding inputs have different dimensions: $\Q\in\RR^{4{,}096\times64}, \K\in\RR^{1{,}024\times64}$ and $\V\in\RR^{1{,}024\times256}$. 
Following the experiments in \cite{chen2021scatterbrain}, we use the pre-trained BigGAN\footnote{\url{https://github.com/huggingface/pytorch-pretrained-BigGAN}} on ImageNet at $512\times512$ resolution and replace the exact attention with its approximations. We generate $5{,}000$ fake images and compute the Frechet Inception Distance (FID) with ImageNet validation set as ground truth and Inception Scores (IS)~\cite{salimans2016improved}. Note that lower FID and higher IS values imply better generation quality. 
We also calculate FLOPS for operations in the attention layer. 
We set the hyperparameters (i.e., feature dimensions) so that all approximation methods have the same peak memory usage.  The results are reported in \cref{tab:biggan}. 
Interestingly, our algorithm shows a lower FID value than the exact attention with $4.14 \times$ fewer FLOPs.  Although Performer is the fastest algorithm, its generated images are unnatural compared while our attention can generate more realistic images. A number of generated images by various methods can be found in the \cref{sec:appendix_biggan}.

\begin{table}[t]
\centering
\caption{Results on image generation using BigGAN with the exact attention and its approximations. Bold values indicate the best within the standard deviation.}\label{tab:biggan}
\vspace{0.1in}
\begin{tabular}{@{}lcccc@{}}
    \toprule
    Method       & FID ($\downarrow$) & IS ($\uparrow$) & \multicolumn{2}{c}{GFLOPS} \\ \midrule
    Exact        & 32.17 & {\bf 58.38} $\pm$ 4.23 & 10.738 & $-$\\
    Reformer     & 72.39 & 19.04 $\pm$ 2.32 & 10.872 & (0.99$\times$)\\
    Performer    & 33.39 & 37.32 $\pm$ 2.91 & {\bf 1.682} & (6.38$\times$) \\
    ScatterBrain & 38.55 & 36.43 $\pm$ 3.34 & 2.891 & (3.71$\times$) \\
    KDEformer & {\bf  31.41} & {\bf 58.16} $\pm$ 4.04 & 2.596 & (4.14$\times$) \\ \bottomrule
\end{tabular}
\end{table}

\begin{table}[t]
\centering
\caption{Results on ImageNet classification using T2T-ViT with the exact attention and its approximations.}\label{tab:t2tvit}
\vspace{0.1in}
\begin{tabular}{@{}lccc@{}}
    \toprule
    Method       & Top-1 Accuracy (\%) & \multicolumn{2}{c}{GFLOPS}  \\ \midrule
    Exact        & {\bf 82.55} & 161.10 & $-$\\
    Reformer     & 81.44 & 11.71 & (13.75 $\times$)\\
    Performer    & 80.50 & {\bf 5.06} & (31.87 $\times$)\\
    ScatterBrain & 81.95 & 7.18 & (22.43 $\times$) \\
    KDEformer         & 82.08 & 8.80 & (18.30 $\times$) \\ \bottomrule
\end{tabular}
\end{table}

\vspace{-0.08in}
\subsection{ImageNet classification with Vision Transformer} \label{sec:vit}
Finally, we evaluate the attention approximations on image classification with Tokens-to-Token Vision Transformer\footnote{\url{https://github.com/yitu-opensource/T2T-ViT}}~\cite{yuan2021tokens}. 
The model consists of Tokens-to-Token (T2T) module and the Vision Transformer (ViT) backbone where the computational bottleneck comes from the T2T module.
Again, we use the pre-trained model with 24 layers in ViT backbone and apply our method to 2 attention layers in the T2T module as a drop-in replacement. 
The dimensions of $\Q,\K,\V$ are all the same, $n=3{,}136,d=64$ in the first layer and $n=784,d=64$ in the second layer. We compute top-1 accuracy on ImageNet validation dataset and measure FLOPS in the first attention layer, which requires the most resources. The results are shown in \cref{tab:t2tvit}. Observe that our method is the best among all approximate methods with $82.08\%$ test accuracy. In particular, it leads to less than $1\%$ performance drop compared to the exact computation but the required operations are $18.3\times$ fewer. Such performance gains would increase when token sequence lengths are larger.

\begin{table*}[t]
    \centering
    \vspace{-0.1in}
    \caption{Results on end-to-end training on $5$ Long Range Arena (LRA) benchmark datasets.}\label{tab:lra}
    \vspace{0.05in}
    \centering
    \scalebox{0.93}{
    \begin{tabular}{@{}lcccccc@{}}
	\toprule
	& ListOps  & Text     & Image    & Retrieval & Pathfinder & Average \\ \midrule 
	Exact     & 33.32  & 60.22 & 37.41 & 81.07 & 70.25 & 56.45 \\
	Reformer  & 36.74  & 61.39 & 43.59 & 78.15 & 66.25 & {\bf 57.22} \\
	Performer & 37.75  & 58.81 & 35.74 & 80.39 & 62.84 & 55.11 \\
	KDEformer & 36.64  & 62.00 & 45.45 & 73.52 & 68.13 & 57.15 \\
	\bottomrule
    \end{tabular}
    }
    \vspace{-0.025in}
    \begin{center}
       (a) Test accuracy (\%) 
    \end{center}
    \scalebox{0.93}{
    \begin{tabular}{@{}lcccccc@{}}
        \toprule
                  & ListOps  & Text     & Image    & Retrieval & Pathfinder & Average  \\ \midrule 
        Exact     & 6.53 & 16.71 & 9.41 & 8.72 & 4.70 & 9.21 \\
        Reformer  & 1.59 & 3.18 & 6.36 & 2.94 & 3.18 & 3.45\\
        Performer & 1.07 & 2.13 & 4.28 & 2.15 & 2.14 & 2.35 \\
        KDEformer & 1.02 & 2.03 & 4.08 & 2.38 & 1.87 & {\bf 2.28} \\
        \bottomrule
    \end{tabular}
    }
    \vspace{-0.025in}
    \begin{center}
        (b) Peak memory (GB)
    \end{center}
    \scalebox{0.93}{
    \begin{tabular}{@{}lcccccc@{}}
        \toprule
                  & ListOps  & Text     & Image    & Retrieval & Pathfinder & Average  \\ \midrule 
        Exact     & 0.133 & 0.479 & 0.276 & 0.478 & 0.141 & 0.301 \\
        Reformer  & 0.041 & 0.081 & 0.155 & 0.092 & 0.082 & 0.090 \\
        Performer & 0.036 & 0.067 & 0.127 & 0.074 & 0.068 & 0.074 \\
        KDEformer & 0.034 & 0.058 & 0.110 & 0.073 & 0.063 & {\bf 0.068} \\
        \bottomrule
    \end{tabular}
    }
    \vspace{-0.025in}
    \begin{center}
        (c) Wall-clock time (sec) per batch
    \end{center}
\end{table*}

\subsection{End-to-end Training with Long Range Arena Benckmark} \label{sec:end}
Finally, to demonstrate the power of our method in reducing the training time of transformer models, we run end-to-end training on the Long Range Arena benchmark~\cite{tay2020long}, which contains $5$ classification datasets, i.e., ListOps, Text, Image, Retrieval and Pathfinder. The maximum sequence lengths of these datasets are $2{,}048$, $4{,}096$, $1{,}024$, $4{,}096$ and $1{,}024$, respectively.
We follow the same settings from \cite{chen2021skyformer}; model is a $2$-layer transformer with $64$ embedding dimension, $128$ hidden dimension, $2$ attention heads, and mean pooling is used for the classification task. Learning rate is set to $10^{-4}$ for Text, ListOps, Image and $2\times 10^{-4}$ for the rest. All models are trained for $50{,}000$ steps. Similar to \cref{sec:single}, we choose hyperparameters of all methods having equal feature dimensions to $128$.

In \cref{tab:lra}, we provide results on (a) test accuracy, (b) peak memory and (c) wall-clock time per batch of single training step (including forward and backward propagations). As a result, we observe that the proposed KDEformer achieves the second-best test accuracy in average followed by Reformer, but it requires much less memory as well as faster wall-clock time than other competitors. For example, KDEformer with Text dataset runs about 8$\times$ faster than the exact attention.

\section{Conclusion}

We propose a fast attention approximation based on recent advances in KDE solvers.
The proposed algorithm can run in strongly sub-quadratic time in sequence length and provide an error bound under the spectral norm. 
It shows promising performances under various practical applications involving long-sequence attention. We believe this can have a significant impact on other practical problems as well.

\section{Acknowledgement}
We would like to thank Navid Nouri for his helpful ideas and discussions about new advancements in kernel density estimation and their potential application.
Amir Zandieh was supported by the Swiss NSF grant No. P2ELP2\textunderscore 195140. Amin Karbasi acknowledges funding in direct support of this work from NSF (IIS-1845032), ONR (N00014- 19-1-2406), and the AI Institute for Learning-Enabled Optimization at Scale (TILOS).

\bibliographystyle{plainnat}
\bibliography{references}

\section{Practical Angular LSH with Fixed Bucket Sizes}\label{appendix_lsh_design}
The practical version of our algorithm that we presented in \cref{sec:practical_improvement} requires a locality sensitive hashing $\Hcal: \RR^d \to [B]$ for identifying the dominant entries of the attention matrix $\A$, which correspond to pairs of keys and queries whose ``angular distances'' are small. 
In this section, we develop a simple yet effective and practical LSH function whose collision probability is related to the angular distance between hashed points.
	
While the lsh allows computing a very sparse approximation to the attention matrix, uneven bucket sizes hinder batching of the computations across lsh buckets. In fact, if we parallelize the computation across buckets, the largest bucket determines the runtime~\cite{kitaev2019reformer}.
Our proposed lsh function has equal-sized buckets, thus, it aligns with modern hardware's block-memory access and can be efficiently parallelized by batching across buckets.
	
We start by defining a simple LSH function whose collision probability is \emph{roughly} proportional to the angle between the hashed points.
	
\begin{defn}[Angular LSH]\label{def:angular_lsh}
    For positive integers $d, r$, let $w_1, w_2, \ldots w_r$ be i.i.d. random samples from the tropical Gaussian distribution $\Ncal(0, I_d)$. We define the \emph{rank-$r$ angular LSH} $h: \RR^d \to \{0, 1\}^r$ as follows:
    \[
    h(x) := \left( \mathbbm{1}_{ \{w_1^\top x\} }, \mathbbm{1}_{ \{w_2^\top x\} }, \ldots \mathbbm{1}_{ \{w_r^\top x\} } \right) ~~~~~\text{ for any } x \in \RR^d.
    \]
    Note that the buckets are labeled by $r$-bit binary numbers and if $r \le d$ then almost surely the total number of buckets is $2^r$.
\end{defn}

It is easy to calculate the collision probability of the angular lsh defined in \cref{def:angular_lsh}.
\begin{claim}
    For positive integers $r, d$ let $h(\cdot)$ be an instance of rank-$r$ angular LSH as per \cref{def:angular_lsh}. For any $x,y \in \RR^d$ the collision probability of $h(x)$ and $h(y)$ is:
    \[
    \Pr[h(x) = h(y)] = \left( 1 - \frac{\theta_{x,y}}{\pi} \right)^r,
    \]
    where $\theta_{x,y} = \cos^{-1} \left( \frac{x^\top y}{\norm{x} \cdot \norm{y}} \right)$ denotes the angle between $x$ and $y$.
\end{claim}

Therefore, the points with small angular distances are likely to be hashed to the same buckets while points with large angular distances are unlikely to be hashed to the same buckets. 

So, if we hash keys $k_j$ and queries $q_i$ using the angular lsh given in \cref{def:angular_lsh} then the entries of the attention matrix $\A$ which correspond to colliding pairs of keys and queries will likely have very large values.
As we mentioned earlier, the main efficiency bottleneck in this lsh-based approach for computing the dominant entries of the attention matrix is the unevenness of hash bucket sizes. 
If we try to compute the sparse approximation to $\A$, as defined in \cref{eq:A_sparse}, using the lsh function from \cref{def:angular_lsh} by parallelizing the computation across buckets, the runtime will be dominated by the time to compute entries in the largest bucket.

One solution for increasing efficiency, which was proposed in \cite{kitaev2019reformer}, is to truncate the lsh buckets and force them to contain equal number of keys and queries. 
However, truncation can degrade the quality of approximation drastically because there will be \emph{spillover} from one bucket to another, and some points can be forced into far-away buckets.
The reason for this spillover effect is the fact that consecutive buckets in a hash table do not necessarily represent areas of the $\RR^d$ space which are geometrically close to each other.

We show that in fact, it is possible to sort the buckets of the angular lsh from \cref{def:angular_lsh} such that the order of buckets reflects their geometrical position, thus, consecutive buckets actually represent neighboring partitions of $\RR^d$. 
It turns out that the geometric distance between two buckets of this lsh function translates into the Hamming distance between their binary labels.

To be precise, for any binary numbers $b_1, b_2 \in \{0,1\}^r$ let $d_H(b_1, b_2) \in [r+1]$ represent the \emph{Hamming distance} between the two, i.e., the number of bits where $b_1$ and $b_2$ differ.
Now note that the lsh buckets in \cref{def:angular_lsh} are labeled with $r$-bit binary numbers. 
Each bit in the binary representations of buckets corresponds to a partitioning of the $\RR^d$ into two sides of a random hyperplane whose normal vector is sampled from a tropical Gaussian.
Therefore, if we have two buckets $b_1$ and $b_2$ with hamming distance $d_H(b_1, b_2) = 1$ then these buckets are positioned on the same sides of all random hyperplanes except for one, thus, they represent neighboring regions in $\RR^d$ and the hyperplanes corresponding to the differing bit of $b_1$ and $b_2$ is the boundary between two regions.

We show this fact in \cref{fig:lsh_partition}, which illustrates the space partitions corresponding to the buckets of a rank-$2$ angular lsh in dimension $d=2$. It is clearly visible that the bucket labels of neighboring partitions have unit Hamming distance.
In \cref{fig:lsh_init_hashing} we hash an example dataset using this LSH function and as can be seen, the buckets have uneven sizes.
Because of the relationship between the Hamming distance of bucket labels and the distance between space partitions, if we order the dataset according to the Hamming ordering of their buckets and then truncate them we get new buckets with even sizes and minimal spillover effect. 
In particular, in \cref{fig:lsh_example_truncation} we order the dataset such that the points from buckets $00, 01, 11, 10$ come in this specific order and then we bin the data points by partitioning the ordered dataset into equal-sized parts. 
The resulting bins show no spillover effect.

\begin{figure}[t]
    \centering
    \subfigure[Space partitions by angular LSH]{\label{fig:lsh_partition}
        \begin{tikzpicture}[scale=0.63]
            \begin{axis}[
                xmin=-1, xmax=1,
                ymin=-1, ymax=1]
                
                \addplot[name path = A, -latex, samples = 1000]{-1*x} node [very near end, right] {$y=x/3$};
                
                \addplot[name path = B, -latex, samples = 1000]{x/3} node [very near end, right] {$y=-x$};
                
                \fill[violet!10](-1,-1) -- (-1,-0.33333) -- (0,0) -- (1,-1);
                \fill[violet!40](1,-1) -- (0,0) -- (1,0.3333);
                \fill[violet!70](-1,-0.33333) -- (0,0) -- (-1,1);
                \fill[violet!100](-1,1) -- (0,0) -- (1,0.33333) -- (1, 1);
                \node[] at (0,-0.5) {\LARGE{$00$}};
                \node at (0,0.5) {\LARGE{$11$}};
                \node at (0.8,0) {\LARGE{$01$}};
                \node at (-0.8,0) {\LARGE{$10$}};
                
            \end{axis}
            
        \end{tikzpicture}
    }
    \subfigure[Hashing an example dataset]{\label{fig:lsh_init_hashing}
        \begin{tikzpicture}[scale=0.63]
            \begin{axis}[
                xmin=-1, xmax=1,
                ymin=-1, ymax=1]
                
                \addplot[name path = A, -latex, samples = 1000, thick]{-1*x} node [very near end, right] {$y=x/3$};
                
                \addplot[name path = B, -latex, samples = 1000, thick]{x/3} node [very near end, right] {$y=-x$};
                
                \node at (-0.72,-0.37) [place_b1] {};
                \node at (-0.8,-0.41) [place_b1] {};
                \node at (-0.39, -0.8) [place_b1] {};
                \node at (-0.38, -0.67) [place_b1] {};
                \node at (-0.0, -0.63) [place_b1] {};
                \node at (0.3,-0.67) [place_b1] {};
                \node at (0.4, -0.72) [place_b1] {};
                \node at (0.52, -0.67) [place_b1] {};
                \node at (0.6,0.08) [place_b2] {};
                \node at (0.69, -0.43) [place_b2] {};
                \node at (0.75, -0.243) [place_b2] {};
                \node at (0.6, -0.2) [place_b2] {};
                \node at (0.55, -0.43) [place_b2] {};
                \node at (-0.72, 0.33) [place_b3] {};
                \node at (-0.8, -0.1) [place_b3] {};
                \node at (-0.8, 0.6) [place_b3] {};
                \node at (0.5,0.6) [place_b4] {};
                \node at (-0.32,0.8) [place_b4] {};
                \node at (0.73,0.412) [place_b4] {};
                \node at (0.2,0.543) [place_b4] {};
                
            \end{axis}
            
        \end{tikzpicture}
    }
    \subfigure[Buckets truncation in Hamming distance order]{\label{fig:lsh_example_truncation}
        \begin{tikzpicture}[scale=0.63]
            \begin{axis}[
                xmin=-1, xmax=1,
                ymin=-1, ymax=1]
                
                \addplot[name path = A, -latex, samples = 1000, dotted]{-1*x} node [very near end, right] {$y=x/3$};
                
                \addplot[name path = B, -latex, samples = 1000, dotted]{x/3} node [very near end, right] {$y=-x$};
                
                \node at (-0.72,-0.37) [place_b1] {};
                \node at (-0.8,-0.41) [place_b1] {};
                \node at (-0.39, -0.8) [place_b1] {};
                \node at (-0.38, -0.67) [place_b1] {};
                \node at (-0.0, -0.63) [place_b1] {};
                \node at (0.3,-0.67) [place_b2] {};
                \node at (0.4, -0.72) [place_b2] {};
                \node at (0.52, -0.67) [place_b2] {};
                \node at (0.69, -0.43) [place_b2] {};
                \node at (0.55, -0.43) [place_b2] {};
                \node at (0.6,0.08) [place_b4] {};
                \node at (0.75, -0.243) [place_b4] {};
                \node at (0.6, -0.2) [place_b4] {};
                \node at (-0.72, 0.33) [place_b3] {};
                \node at (-0.8, -0.1) [place_b3] {};
                \node at (-0.8, 0.6) [place_b3] {};
                \node at (0.5,0.6) [place_b4] {};
                \node at (-0.32,0.8) [place_b3] {};
                \node at (0.73,0.412) [place_b4] {};
                \node at (0.2,0.543) [place_b3] {};
                
                \addplot[black,thick] coordinates { (0,0) (0.1,-1) };
                \addplot[black,thick] coordinates { (0,0) (1,-0.45) };
                \addplot[black,thick] coordinates { (0,0) (0.5,1) };
                \addplot[black,thick] coordinates { (0,0) (-1,-0.3333) };
                
            \end{axis}
            
        \end{tikzpicture}
    }
    \caption{Rank-$2$ Angular LSH in action (in dimension $d=2$). The space partitions corresponding to buckets with unit Hamming distance are neighbors in $\RR^d$. In \cref{fig:lsh_init_hashing} we hash an example dataset and we get uneven buckets. \cref{fig:lsh_example_truncation} show that if we order the dataset according to the Hamming distance of their buckets and then truncate the buckets we get new equal-sized buckets with minimal spillover effect.}
    \label{fig:lsh_examples}
\end{figure}
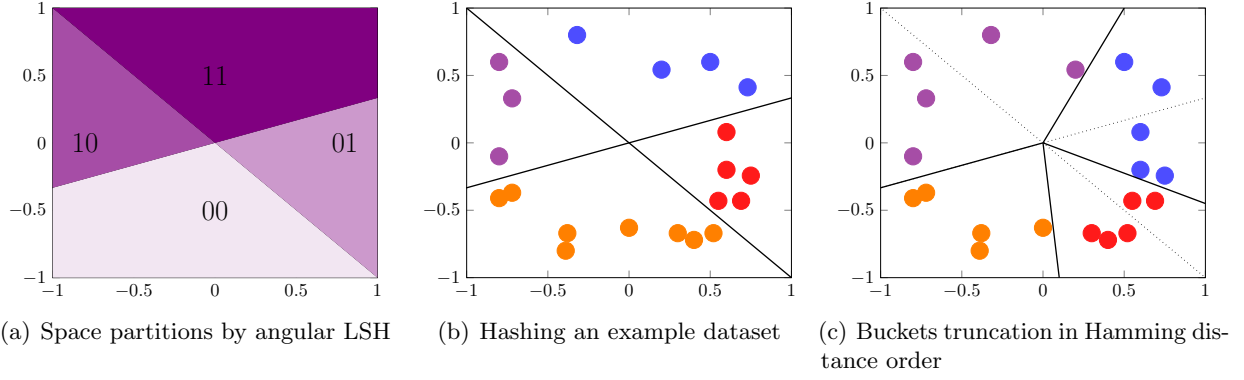

In the following lemma we show how to order $r$-bit binary numbers $\{ 0, 1 \}^r$ such that all consecutive numbers have unit Hamming distance:
\begin{lemma}[Ordering of binary numbers according to their Hamming distance]\label{lem_binary_ordering_hamming}
    For any positive integer $r$ it is possible to order the set of binary numbers $\{ 0, 1 \}^r$ as a sequence $b_1, b_2, \ldots b_{2^r}$ such that for any $j \in [2^r-1]$:
    \[
    d_H(b_j, b_{j+1}) = 1.
    \]
\end{lemma}
\begin{proof}
    The proof is by induction. For $r=1$ the base of induction follows trivially.
    Now suppose that we have the sequence of $(r-1)$-bit numbers $b'_1, b'_2, \ldots b'_{2^{r-1}}$ such that $d_H(b'_j, b'_{j+1}) = 1$ for any $j \in [2^{r-1}-1]$. Then the sequence of $r$-bit numbers will be as follows:
    \[
    b_j := \begin{cases}
        (b'_j , ~0) & \text{ if } j \le 2^{r-1}\\
        (b'_{2^r+1-j}, ~1) & \text{ if } j > 2^{r-1}
    \end{cases} ~~~~~~~\text{ for } j \in [2^r].
    \]
    One can verify that this sequence satisfies the desired property and the proof is complete.
\end{proof}

Therefore, we can use the angular LSH together with the ordering of binary numbers from \cref{lem_binary_ordering_hamming} to construct an effective hash function with equal-sized buckets. 
	
\begin{defn}[Equal-sized LSH with Minimal Spillover]\label{def_lsh_truncation_no_spillover}
    Suppose that we want to hash a dataset $x_1, x_2, \ldots x_n \in \RR^d$.
    \begin{enumerate}
        \item Hash these points using a rank-$r$ Angular LSH $h(\cdot)$ as per \cref{def:angular_lsh}. 
        \item Then, using \cref{lem_binary_ordering_hamming}, produce an ordering of $r$-bit binary numbers such that consecutive numbers have unit Hamming distance; let $b_1, b_2, \ldots b_{2^r}$ be such ordering.
        \item Next, define a permutation $\Pcal \in {\tt Sym}(n)$ which orders the dataset according to the Hamming ordering of their buckets. 
        More specifically, $\Pcal$ satisfies:
        \[
        \Pcal(i) < \Pcal(j) ~~~\text{ iff } h(x_i) \le_{*} h(x_j) \text{, where the inequality $\le_{*}$ acts with respect to the ordering } b_1, b_2, \ldots b_{2^r}.
        \]
        \item Permute $x_1, x_2, \ldots x_n$ according to $\Pcal$ and then partition the sequence into equal-sized chunks. These chunks are the buckets.
    \end{enumerate}
\end{defn}

\begin{figure}[t]
    \centering
    \begin{tikzpicture}[every node/.style={scale=0.8}, text depth=.5ex,text height=1ex,text width=1.5ex]
        
        \matrix (A_before) [black2darr]{
            |[fill=orange!70]|\phantom{} & |[fill=orange!10]|\phantom{} &|[fill=orange!10]|\phantom{} & |[fill=orange!10]|\phantom{} & |[fill=orange!100]|\phantom{} & |[fill=orange!10]|\phantom{} & |[fill=orange!10]|\phantom{}& |[fill=orange!60]|\phantom{}& |[fill=orange!100]|\phantom{}\\ 
            |[fill=orange!10]|\phantom{} & |[fill=orange!70]|\phantom{} & |[fill=orange!10]|\phantom{} & |[fill=orange!70]|\phantom{} & |[fill=orange!30]|\phantom{} & |[fill=orange!10]|\phantom{} & |[fill=orange!30]|\phantom{} & |[fill=orange!100]|\phantom{}& |[fill=orange!10]|\phantom{}\\ 
            |[fill=orange!10]|\phantom{} & |[fill=orange!30]|\phantom{} & |[fill=orange!100]|\phantom{} &|[fill=orange!10]|\phantom{} & |[fill=orange!10]|\phantom{} & |[fill=orange!70]|\phantom{} & |[fill=orange!30]|\phantom{} & |[fill=orange!10]|\phantom{}& |[fill=orange!40]|\phantom{}\\ 
            |[fill=orange!10]|\phantom{} & |[fill=orange!30]|\phantom{} & |[fill=orange!40]|\phantom{} & |[fill=orange!10]|\phantom{} & |[fill=orange!10]|\phantom{} & |[fill=orange!70]|\phantom{} & |[fill=orange!100]|\phantom{} & |[fill=orange!30]|\phantom{}& |[fill=orange!10]|\phantom{}\\ 
            |[fill=orange!10]|\phantom{} & |[fill=orange!100]| \phantom{} & |[fill=orange!30]|\phantom{} & |[fill=orange!100]| \phantom{} & |[fill=orange!10]|\phantom{} & |[fill=orange!10]|\phantom{} & |[fill=orange!30]|\phantom{} & |[fill=orange!70]|\phantom{}& |[fill=orange!10]|\phantom{}\\ 
            |[fill=orange!10]|\phantom{} & |[fill=orange!10]|\phantom{} & |[fill=orange!70]|\phantom{} & |[fill=orange!10]|\phantom{} & |[fill=orange!10]|\phantom{} & |[fill=orange!100]|\phantom{} & |[fill=orange!70]|\phantom{}& |[fill=orange!40]|\phantom{}& |[fill=orange!10]|\phantom{}\\ 
            |[fill=orange!70]|\phantom{} & |[fill=orange!70]|\phantom{} & |[fill=orange!10]|\phantom{} & |[fill=orange!10]|\phantom{} & |[fill=orange!100]|\phantom{} & |[fill=orange!10]|\phantom{} & |[fill=orange!10]|\phantom{}& |[fill=orange!10]|\phantom{}& |[fill=orange!100]|\phantom{}\\ 
            |[fill=orange!10]|\phantom{} & |[fill=orange!100]|\phantom{} & |[fill=orange!10]|\phantom{} & |[fill=orange!100]|\phantom{} & |[fill=orange!10]|\phantom{} & |[fill=orange!10]|\phantom{} & |[fill=orange!30]|\phantom{} &  |[fill=orange!70]|\phantom{}& |[fill=orange!10]|\phantom{}\\ 
            |[fill=orange!40]|\phantom{} & |[fill=orange!30]|\phantom{} & |[fill=orange!30]|\phantom{} & |[fill=orange!10]|\phantom{} & |[fill=orange!70]|\phantom{} & |[fill=orange!10]|\phantom{} & |[fill=orange!10]|\phantom{}& |[fill=orange!50]|\phantom{}& |[fill=orange!70]|\phantom{}\\ 
        };
        \node[below=3.5em of A_before-7-4] {\Large $~~\A$};
        
        \matrix (unsorted_left_before) [black2darr, left=0.2em of A_before] {
            |[fill=violet!100]|\phantom{} \\ 
            |[fill=violet!10]|\phantom{} \\ |[fill=violet!70]|\phantom{} \\ |[fill=violet!40]|\phantom{} \\ |[fill=violet!10]|\phantom{} \\ |[fill=violet!70]|\phantom{} \\ |[fill=violet!100]|\phantom{}\\
            |[fill=violet!10]|\phantom{}\\
            |[fill=violet!100]|\phantom{}\\
        };
        \node[right=-0.3em of unsorted_left_before-1-1] {\small{$q_1$}};
        \node[right=-0.3em of unsorted_left_before-2-1] {\small{$q_2$}};
        \node[right=-0.3em of unsorted_left_before-3-1] {\small{$q_3$}};
        \node[right=-0.3em of unsorted_left_before-4-1] {\small{$q_4$}};
        \node[right=-0.3em of unsorted_left_before-5-1] {\small{$q_5$}};
        \node[right=-0.3em of unsorted_left_before-6-1] {\small{$q_6$}};
        \node[right=-0.3em of unsorted_left_before-7-1] {\small{$q_7$}};
        \node[right=-0.3em of unsorted_left_before-8-1] {\small{$q_8$}};
        \node[right=-0.3em of unsorted_left_before-9-1] {\small{$q_9$}};
        
        \matrix (unsorted_up_before) [black2darr, above=0.2em of A_before] {
            |[fill=violet!100]|\phantom{} & |[fill=violet!10]|\phantom{} & |[fill=violet!70]|\phantom{} & |[fill=violet!10]|\phantom{} & |[fill=violet!100]|\phantom{} & |[fill=violet!70]|\phantom{} & |[fill=violet!40]|\phantom{} & |[fill=violet!10]|\phantom{} & |[fill=violet!100]|\phantom{}\\
        };
        \node[below=-0.12em of unsorted_up_before-1-1] {\small{$k_1$}};
        \node[below=-0.12em of unsorted_up_before-1-2] {\small{$k_2$}};
        \node[below=-0.12em of unsorted_up_before-1-3] {\small{$k_3$}};
        \node[below=-0.12em of unsorted_up_before-1-4] {\small{$k_4$}};
        \node[below=-0.12em of unsorted_up_before-1-5] {\small{$k_5$}};
        \node[below=-0.12em of unsorted_up_before-1-6] {\small{$k_6$}};
        \node[below=-0.12em of unsorted_up_before-1-7] {\small{$k_7$}};
        \node[below=-0.12em of unsorted_up_before-1-8] {\small{$k_8$}};
        \node[below=-0.12em of unsorted_up_before-1-9] {\small{$k_9$}};
        
        \node[right=0.2em of A_before] (plus) {\LARGE $\Rightarrow$};\node[above=0.2em of plus] (plus) {\LARGE $\mathcal{P}$};
        
        \matrix (A_after) [black2darr, right=4.5em of A_before] {
            |[fill=orange!70]|\phantom{} & |[fill=orange!100]|\phantom{} & |[fill=orange!70]|\phantom{} & |[fill=orange!30]|\phantom{} & |[fill=orange!10]|\phantom{} & |[fill=orange!10]|\phantom{} & |[fill=orange!10]|\phantom{} & |[fill=orange!10]|\phantom{} & |[fill=orange!30]|\phantom{}\\|[fill=orange!100]|\phantom{} & |[fill=orange!70]|\phantom{} & |[fill=orange!100]|\phantom{} & |[fill=orange!30]|\phantom{} & |[fill=orange!10]|\phantom{} & |[fill=orange!10]|\phantom{} & |[fill=orange!10]|\phantom{} & |[fill=orange!10]|\phantom{} & |[fill=orange!10]|\phantom{}\\|[fill=orange!100]|\phantom{} & |[fill=orange!70]|\phantom{} & |[fill=orange!100]|\phantom{} & |[fill=orange!30]|\phantom{} & |[fill=orange!30]|\phantom{} & |[fill=orange!10]|\phantom{} & |[fill=orange!10]|\phantom{} & |[fill=orange!10]|\phantom{} & |[fill=orange!10]|\phantom{}\\|[fill=orange!30]|\phantom{} & |[fill=orange!30]|\phantom{} & |[fill=orange!10]|\phantom{} & |[fill=orange!100]|\phantom{} & |[fill=orange!40]|\phantom{} & |[fill=orange!70]|\phantom{} & |[fill=orange!10]|\phantom{} & |[fill=orange!10]|\phantom{} & |[fill=orange!10]|\phantom{}\\|[fill=orange!30]|\phantom{} & |[fill=orange!10]|\phantom{} & |[fill=orange!10]|\phantom{} & |[fill=orange!30]|\phantom{} & |[fill=orange!100]|\phantom{} & |[fill=orange!70]|\phantom{} & |[fill=orange!40]|\phantom{} & |[fill=orange!10]|\phantom{} & |[fill=orange!10]|\phantom{}\\|[fill=orange!10]|\phantom{} & |[fill=orange!40]|\phantom{} & |[fill=orange!10]|\phantom{} & |[fill=orange!70]|\phantom{} & |[fill=orange!70]|\phantom{} & |[fill=orange!100]|\phantom{} & |[fill=orange!10]|\phantom{} & |[fill=orange!10]|\phantom{} & |[fill=orange!10]|\phantom{}\\|[fill=orange!70]|\phantom{} & |[fill=orange!10]|\phantom{} & |[fill=orange!10]|\phantom{} & |[fill=orange!10]|\phantom{} & |[fill=orange!10]|\phantom{} & |[fill=orange!10]|\phantom{} & |[fill=orange!100]|\phantom{} & |[fill=orange!70]|\phantom{} & |[fill=orange!100]|\phantom{}\\|[fill=orange!30]|\phantom{} & |[fill=orange!50]|\phantom{} & |[fill=orange!10]|\phantom{} & |[fill=orange!10]|\phantom{} & |[fill=orange!30]|\phantom{} & |[fill=orange!10]|\phantom{} & |[fill=orange!70]|\phantom{} & |[fill=orange!40]|\phantom{} & |[fill=orange!70]|\phantom{}\\|[fill=orange!10]|\phantom{} & |[fill=orange!60]|\phantom{} & |[fill=orange!10]|\phantom{} & |[fill=orange!10]|\phantom{} & |[fill=orange!10]|\phantom{} & |[fill=orange!10]|\phantom{} & |[fill=orange!100]|\phantom{} & |[fill=orange!70]|\phantom{} & |[fill=orange!100]|\phantom{}\\
        };
        \node[below=3.5em of A_after-7-4] {\Large $~~~\A_\Pcal$};
        
        \matrix (sorted_left_after) [black2darr, left=0.2em of A_after] {
            |[fill=violet!10]|\phantom{} \\ 
            |[fill=violet!10]|\phantom{} \\ |[fill=violet!10]|\phantom{} \\ |[fill=violet!40]|\phantom{} \\ |[fill=violet!40]|\phantom{} \\ |[fill=violet!70]|\phantom{} \\ |[fill=violet!100]|\phantom{}\\
            |[fill=violet!100]|\phantom{}\\
            |[fill=violet!100]|\phantom{}\\
        };
        
        \node[right=-0.3em of sorted_left_after-1-1] {\small{$q_2$}};
        \node[right=-0.3em of sorted_left_after-2-1] {\small{$q_8$}};
        \node[right=-0.3em of sorted_left_after-3-1] {\small{$q_5$}};
        \node[right=-0.3em of sorted_left_after-4-1] {\small{$q_4$}};
        \node[right=-0.3em of sorted_left_after-5-1] {\small{$q_3$}};
        \node[right=-0.3em of sorted_left_after-6-1] {\small{$q_6$}};
        \node[right=-0.3em of sorted_left_after-7-1] {\small{$q_7$}};
        \node[right=-0.3em of sorted_left_after-8-1] {\small{$q_9$}};
        \node[right=-0.3em of sorted_left_after-9-1] {\small{$q_1$}};
        
        \matrix (sorted_up_after) [black2darr, above=0.2em of A_after] {
            |[fill=violet!10]|\phantom{} & |[fill=violet!10]|\phantom{} & |[fill=violet!10]|\phantom{} & |[fill=violet!40]|\phantom{} & |[fill=violet!40]|\phantom{} & |[fill=violet!70]|\phantom{} & |[fill=violet!100]|\phantom{} &|[fill=violet!100]|\phantom{} & |[fill=violet!100]|\phantom{}\\
        };
        
        \node[below=-0.12em of sorted_up_after-1-1] {\small{$k_2$}};
        \node[below=-0.12em of sorted_up_after-1-2] {\small{$k_8$}};
        \node[below=-0.12em of sorted_up_after-1-3] {\small{$k_4$}};
        \node[below=-0.12em of sorted_up_after-1-4] {\small{$k_7$}};
        \node[below=-0.12em of sorted_up_after-1-5] {\small{$k_3$}};
        \node[below=-0.12em of sorted_up_after-1-6] {\small{$k_6$}};
        \node[below=-0.12em of sorted_up_after-1-7] {\small{$k_9$}};
        \node[below=-0.12em of sorted_up_after-1-8] {\small{$k_1$}};
        \node[below=-0.12em of sorted_up_after-1-9] {\small{$k_5$}};
        
        \node[right=0.2em of A_after] (right_arrow) {\LARGE $\Rightarrow$}; \node[above=0.2em of right_arrow] (plus) {\LARGE $\mathcal{P}^{-1}$};
        
        \matrix (A_primary) [black2darr, right=3.5em of A_after]{
            |[fill=orange!70]|\phantom{} & |[fill=orange!10]|\phantom{} &|[fill=orange!10]|\phantom{} & |[fill=orange!10]|\phantom{} & |[fill=orange!100]|\phantom{} & |[fill=orange!10]|\phantom{} & |[fill=orange!10]|\phantom{}& |[fill=orange!60]|\phantom{}& |[fill=orange!100]|\phantom{}\\ 
            |[fill=orange!10]|\phantom{} & |[fill=orange!70]|\phantom{} & |[fill=orange!10]|\phantom{} & |[fill=orange!70]|\phantom{} & |[fill=orange!30]|\phantom{} & |[fill=orange!10]|\phantom{} & |[fill=orange!30]|\phantom{} & |[fill=orange!100]|\phantom{}& |[fill=orange!10]|\phantom{}\\ 
            |[fill=orange!10]|\phantom{} & |[fill=orange!30]|\phantom{} & |[fill=orange!100]|\phantom{} &|[fill=orange!10]|\phantom{} & |[fill=orange!10]|\phantom{} & |[fill=orange!70]|\phantom{} & |[fill=orange!30]|\phantom{} & |[fill=orange!10]|\phantom{}& |[fill=orange!40]|\phantom{}\\ 
            |[fill=orange!10]|\phantom{} & |[fill=orange!30]|\phantom{} & |[fill=orange!40]|\phantom{} & |[fill=orange!10]|\phantom{} & |[fill=orange!10]|\phantom{} & |[fill=orange!70]|\phantom{} & |[fill=orange!100]|\phantom{} & |[fill=orange!30]|\phantom{}& |[fill=orange!10]|\phantom{}\\ 
            |[fill=orange!10]|\phantom{} & |[fill=orange!100]| \phantom{} & |[fill=orange!30]|\phantom{} & |[fill=orange!100]| \phantom{} & |[fill=orange!10]|\phantom{} & |[fill=orange!10]|\phantom{} & |[fill=orange!30]|\phantom{} & |[fill=orange!70]|\phantom{}& |[fill=orange!10]|\phantom{}\\ 
            |[fill=orange!10]|\phantom{} & |[fill=orange!10]|\phantom{} & |[fill=orange!70]|\phantom{} & |[fill=orange!10]|\phantom{} & |[fill=orange!10]|\phantom{} & |[fill=orange!100]|\phantom{} & |[fill=orange!70]|\phantom{}& |[fill=orange!40]|\phantom{}& |[fill=orange!10]|\phantom{}\\ 
            |[fill=orange!70]|\phantom{} & |[fill=orange!70]|\phantom{} & |[fill=orange!10]|\phantom{} & |[fill=orange!10]|\phantom{} & |[fill=orange!100]|\phantom{} & |[fill=orange!10]|\phantom{} & |[fill=orange!10]|\phantom{}& |[fill=orange!10]|\phantom{}& |[fill=orange!100]|\phantom{}\\ 
            |[fill=orange!10]|\phantom{} & |[fill=orange!100]|\phantom{} & |[fill=orange!10]|\phantom{} & |[fill=orange!100]|\phantom{} & |[fill=orange!10]|\phantom{} & |[fill=orange!10]|\phantom{} & |[fill=orange!30]|\phantom{} &  |[fill=orange!70]|\phantom{}& |[fill=orange!10]|\phantom{}\\ 
            |[fill=orange!40]|\phantom{} & |[fill=orange!30]|\phantom{} & |[fill=orange!30]|\phantom{} & |[fill=orange!10]|\phantom{} & |[fill=orange!70]|\phantom{} & |[fill=orange!10]|\phantom{} & |[fill=orange!10]|\phantom{}& |[fill=orange!50]|\phantom{}& |[fill=orange!70]|\phantom{}\\ 
        };
        \node[below=3.5em of A_primary-7-4] {\Large $~~~\A_{\tt spar}$};
        
        \newcommand{\blackSquareA}[2]{\draw[teal, black, very thick] (A_primary-#2.south west) rectangle (A_primary-#2.north east);}
        \blackSquareA{1-1}{1-1}
        \blackSquareA{1-5}{1-5}
        \blackSquareA{1-9}{1-9}
        \blackSquareA{2-2}{2-2}
        \blackSquareA{2-4}{2-4}
        \blackSquareA{2-8}{2-8}
        \blackSquareA{3-3}{3-3}
        \blackSquareA{3-6}{3-6}
        \blackSquareA{4-7}{4-7}
        \blackSquareA{5-2}{5-2}
        \blackSquareA{5-4}{5-4}
        \blackSquareA{5-8}{5-8}
        \blackSquareA{6-3}{6-3}  
        \blackSquareA{6-6}{6-6}
        \blackSquareA{7-1}{7-1}
        \blackSquareA{7-5}{7-5}  \blackSquareA{7-9}{7-9}
        \blackSquareA{8-2}{8-2}
        \blackSquareA{8-4}{8-4}  \blackSquareA{8-8}{8-8}
        \blackSquareA{9-1}{9-1}
        \blackSquareA{9-5}{9-5}
        \blackSquareA{9-9}{9-9}
        
        \newcommand{\blackSquareABlock}[2]{\draw[teal, black, very thick] (A_after-#1.north west) rectangle (A_after-#2.south east);}
        \blackSquareABlock{1-1}{3-3}
        \blackSquareABlock{4-4}{6-6}
        \blackSquareABlock{7-7}{9-9}
        
        \matrix (b_1) [black2darr, below=1.5em of A_before] {
            |[fill=violet!10]|\phantom{}\\
        }; \node[right=0.2em of b_1-1-1] {$\mathbf b_1$};
        \matrix (b_2) [black2darr, right=2.5em of b_1] {
            |[fill=violet!40]|\phantom{}\\
        };\node[right=0.2em of b_2-1-1] {$\mathbf b_2$};
        \matrix (b_3) [black2darr, right=2.5em of b_2] {
            |[fill=violet!70]|\phantom{}\\
        };\node[right=0.2em of b_3-1-1] {$\mathbf b_3$};
        \matrix (b_4) [black2darr, right=2.5em of b_3] {
            |[fill=violet!100]|\phantom{}\\
        };\node[right=0.2em of b_4-1-1] {$\mathbf b_4$};
        
    \end{tikzpicture}
    \caption{An example of how $\A_{\tt spar}$ can be computed efficiently. (Left) keys and queries are hashed using the angular lsh function. buckets are represented by shades of violet. (Middle) keys and queries are permuted such that their buckets are sorted according to the Hamming distance ordering. Large entries of the permuted attention matrix $\A_\Pcal$ are concentrated around the diagonal blocks, so we compute the diagonal blocks. (Right) the block diagonal approximation to $\A_\Pcal$ is reverse permuted to obtain $\A_{\tt spar}$.}\label{fig:asparse_computation_lsh_truncation}
\end{figure}
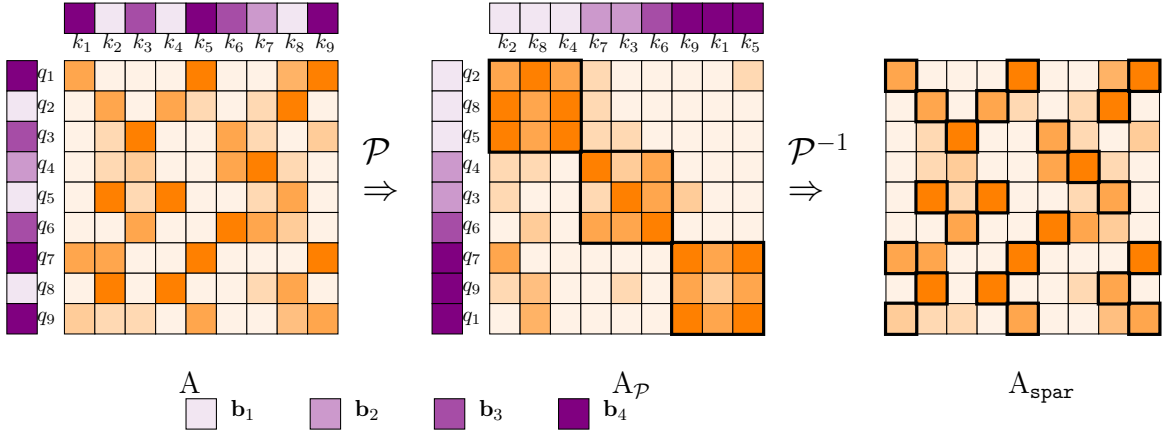

Now we explain how we can use the lsh procedure given in \cref{def_lsh_truncation_no_spillover} to compute $\A_{\tt spar}$ as per \cref{eq:A_sparse} through an example shown in \cref{fig:asparse_computation_lsh_truncation}. We first hash keys $k_j$ and queries $q_i$ via the angular lsh. We represent the buckets of this hashing via different shades of violet in \cref{fig:asparse_computation_lsh_truncation}. 
Clearly, the bucket sizes are uneven. 
Then we permute keys and queries via $\Pcal$ which orders the points such that their buckets are sorted according to the ordering $b_1, b_2, b_3, b_4$ obtained from \cref{lem_binary_ordering_hamming}. 
Then we truncate the sorted points which is in fact equivalent to selecting blocks along the diagonal of the permuted attention matrix. The selected diagonal blocks in \cref{fig:asparse_computation_lsh_truncation} illustrate this. Finally, we can reverse the permutation on the rows and columns of the block diagonal attention which gives us the final $\A_{\tt spar}$.

\section{Omitted Proofs}

\subsection{Proof of \cref{lem:amm-operator-norm}: Approximate Matrix Multiplication via Sampling}\label{appndx-amm-proof}
In this section, we analyze the random sampling method for approximately computing the product of two rectangular matrices, presented in \cref{lem:amm-operator-norm}.
The proof of this lemma is based on the following version of the matrix Bernstein inequality.

\begin{lemma}[Matrix Approximation by Random Sampling, Corollary 6.2.1 from \cite{tropp2015introduction}]\label{lem-matrix-bernstein}
    Let $\B$ be a fixed $q \times d$ matrix.
    Construct a $q \times d$ random matrix $\R$ that satisfies
    \[ \EE[\R] = \B , ~~~~\text{ and }~~~ \norm{\R}_\op \le L. \]
    Compute the per-sample second moment:
    \[ m_2(\R) = \max\{ \norm{\EE[\R^*\R]}_\op, \norm{\EE[\R\R^*]}_\op \}. \]
    Form the matrix sampling estimator
    \[ \overline{\R}_m = \frac{1}{m} \sum_{i=1}^m \R_i ~~~~~\text{ where each $\R_i$ is an independent copy of $\R$}. \]
    Then for every $t>0$, the estimator satisfies
    \[ \Pr\left[ \norm{\overline{\R}_m - \B}_\op \ge t \right] \le (q+d) \cdot \exp\left( \frac{-mt^2/2}{m_2(\R) + 2Lt/3} \right). \]
\end{lemma}

Now we prove \cref{lem:amm-operator-norm} by invoking the above matrix Bernstein inequality.

{\bf Lemma 3.2}~~(Approximate Matrix Multiplication~(AMM)){\bf.}~
{\it 
    For any matrices $\X \in \RR^{n \times q}, \Y \in \RR^{n \times d}$ and any probability distribution $\{ p_i \}_{i \in [n]}$ which satisfies $p_i \ge \frac{1}{4} \cdot \frac{\norm{x_i}_2^2 + \gamma \cdot \norm{y_i}_2^2}{\norm{\X}_F^2 + \gamma \cdot \norm{\Y}_F^2}$ for all $i \in [n]$ and $\gamma = \norm{\X}_\op^2 / \norm{\Y}_\op^2$, a sampling matrix $\bPi \in \RR^{m \times n}$ constructed by first generating $m$ i.i.d. samples $\ell_1, \ell_2, \ldots \ell_m \in [n]$ according to $\{p_\ell\}_{\ell\in[n]}$ and then letting the $r^{th}$ row of $\bPi$ be $\frac{1}{\sqrt{m\cdot p_{\ell_r}}} \cdot e_{\ell_r}^\top$, if $m = \Omega\left( \varepsilon^{-2} \log n \cdot (\sr(\X) + \sr(\Y)) \right)$ for some $\varepsilon>0$, the following holds,
}
\begin{align*}
    \Pr \left[ \normop{\X^\top \bPi^\top \bPi \Y - \X^\top\Y} > \varepsilon\normop{\X} \normop{\Y}\right] \le \frac{1}{\poly(n)}.
\end{align*}

\begin{proof}
    First we let $\B := \X^\top \Y$. Then we let the random matrix $\R$ have the following distribution
    \[ \Pr\left[ \R = \frac{x_i^\top \cdot y_i}{p_i} \right] = p_i ~~~\text{ for }i \in [n] \]
    where $x_i$ and $y_i$ are $i^{th}$ row vector in $\X$ and $\Y$, respectively.
    With this definition we have,
    \[ \EE[\R] = \sum_{i \in [n]} \frac{x_i^\top \cdot y_i}{p_i} \cdot p_i = \sum_{i \in [n]} x_i^\top \cdot y_i = \X^\top\Y = \B. \]
    Furthermore, we can bound the operator norm of $\R$ as follows,
    \begin{align*}
        \norm{\R}_\op &\le \max_{i \in [n]} \frac{\norm{ x_i^\top \cdot y_i}_\op }{p_i} \\
        &= \max_{i \in [n]} \frac{\norm{x_i}_2 \norm{y_i}_2 }{p_i}\\
        &\le 4 \cdot \max_{i \in [n]} \frac{\norm{x_i}_2 \norm{y_i}_2 \cdot \left( \norm{\X}_F^2 + \gamma \cdot \norm{\Y}_F^2 \right)}{ \norm{x_i}_2^2 + \gamma \cdot \norm{y_i}_2^2 }\\
        &\le 2\cdot \max_{i \in [n]} \frac{1}{\sqrt{\gamma}} \cdot \norm{\X}_F^2 + \sqrt{\gamma} \cdot \norm{\Y}_F^2 \\
        &= 2\norm{\X}_\op \cdot \norm{\Y}_\op \cdot\left( \sr(\X) + \sr(\Y) \right) \equiv L,
    \end{align*}
    where the third line above follows from the precondition of \cref{lem:amm-operator-norm} about the distribution $\{ p_i\}_{i\in [n]}$ and the fourth line follows from AM-GM inequality. The last line follows from the definition of $\gamma$ and definition of stable rank.
    Next, we will compute the per-sample second moment as follows,
    \begin{align*}
        \EE[\R^*\R] &= \sum_{i \in [n]} \norm{x_i}_2^2 \cdot \frac{y_i^\top \cdot y_i}{p_i^2} \cdot p_i =  \sum_{i \in [n]} \norm{x_i}_2^2 \cdot \frac{y_i^\top \cdot y_i}{p_i}\\
        &\preceq 4 \cdot \left( \norm{\X}_F^2 + \gamma \cdot \norm{\Y}_F^2 \right) \cdot \sum_{i \in [n]} \frac{\norm{x_i}_2^2}{ \norm{x_i}_2^2 + \gamma \cdot \norm{y_i}_2^2} \cdot y_i^\top y_i \\
        &\preceq 4 \cdot \left( \norm{\X}_F^2 + \gamma \cdot \norm{\Y}_F^2 \right) \cdot \sum_{i \in [n]} y_i^\top y_i = 4 \cdot \left( \norm{\X}_F^2 + \gamma \cdot \norm{\Y}_F^2 \right) \cdot \Y^\top \Y.
    \end{align*}
    Similarly,
    \[ \EE[\R \R^*] \preceq 4 \cdot\left( \norm{\X}_F^2/ \gamma + \norm{\Y}_F^2 \right) \cdot \X^\top \X. \] 
    In summary,
    \begin{align*}
        m_2(\R) &= \max\{ \norm{\EE[\R^*\R]}_\op, \norm{\EE[\R\R^*]}_\op \} \\
        &\le 4 \cdot \max \left\{ \left(\norm{\X}_F^2 + \gamma \cdot \norm{\Y}_F^2 \right) \cdot \norm{ \Y^\top \Y }_\op, \left( \norm{\X}_F^2/ \gamma + \norm{\Y}_F^2 \right) \cdot \norm{\X \X^\top}_\op \right\} \\
        &= 4 \cdot \norm{\X}_\op^2 \norm{\Y}_\op^2 \cdot \left( \sr(\X) + \sr(\Y) \right).
    \end{align*}
    Finally, we note that, from the way the sampling matrix was constructed we have $\X^\top \bPi^\top \bPi \Y = \frac{1}{m} \sum_{r\in [m]} \frac{x_{\ell_r} \cdot y_{\ell_r}}{p_{i_r}} = \overline{\R}_m$.
    Thus, by invoking \cref{lem-matrix-bernstein} we find that for $t = \varepsilon \cdot \norm{\X}_\op \norm{\Y}_\op$ we have,
    \[ \Pr\left[ \norm{\overline{\R}_m - \B}_\op \ge \varepsilon \cdot \norm{\X}_\op \norm{\Y}_\op \right] \le (q+d) \cdot \exp\left( \frac{-mt^2/2}{m_2(\R) + 2Lt/3} \right) \le \frac{1}{\poly(n)}.\]
    This completes the proof of \cref{lem:amm-operator-norm}.
\end{proof}


\subsection{Proof of \cref{thm-correctness-outerloop}}\label{sec:proof-correctness-outer-loop}

\thrmouterloop*

\begin{proof}
    First, note that all entries of $\D^{-1} \A$ are positive and the sum of entries of each row of this matrix equals 1, so by the Gershgorin circle theorem $\norm{\D^{-1} \A}_\op \le 1$. On the other hand, $\D^{-1} \A \cdot \mathbf{1}_n = \mathbf{1}_n$, so we have $\norm{\D^{-1} \A}_\op =1$. We will use this fact in the rest of the proof.
    
    Now note that \cref{alg-outer-loop} computes $\alpha = \textsc{WExpKDE}\left( \frac{\K}{d^{1/4}} , \frac{\Q}{d^{1/4}} , \mathbf{1}_n, \frac{\varepsilon}{3} \right)$ in line~3 and lets $\wt{\D} = \diag(\alpha)$.
    Thus, as we showed earlier, by \cref{def:ExpKDE} and using the fact that entries of $\D$ are positive, we have $(1-\varepsilon/3) \D \preceq \wt{\D} \preceq (1+\varepsilon/3) \D$. 
    So, using this inequality along with the fact that $\norm{\D^{-1}\A }_{\op}=1$, the diagonal matrix $\wt{\D}$ satisfies \cref{eq:error-D-approx}.

    Next, let us consider the vector $\beta = \textsc{WExpKDE}\left( \frac{\sqrt{2} \cdot \Q}{d^{1/4}}, \frac{\sqrt{2} \cdot \K}{d^{1/4}}, u, 1 /3 \right)$ computed in line~4. For ease of notation, let $\X^\top := \wt{\D}^{-1} \A$.
    By \cref{def:ExpKDE} and using the definition of $u_i = 1/\alpha_i^2$ in line~3, we have,
    \[ \beta_j \in (1 \pm 1/3) \cdot \sum_{i \in [n]} u_i \cdot \exp\left( \frac{2}{\sqrt{d}} \langle q_i, k_j \rangle \right) = (1 \pm 1/3) \cdot \norm{x_j}_2^2 ~~~\text{ for any } j \in [n]. \]
    
    Also, note that $\gamma$ which is computed in line~2 of the algorithm is equal to $\gamma = \frac{\norm{\D^{-1} \A}_\op^2}{\norm{\V}_\op^2}$. 
    Because $(1-\varepsilon/3) \D \preceq \wt{\D} \preceq (1+\varepsilon/3) \D$, we have $\gamma \in (1 \pm \varepsilon/3)^{-1} \cdot \tilde{\gamma}$, where $\tilde{\gamma} := \norm{\wt{\D}^{-1} \A}_\op^2 / \norm{\V}_\op^2 $.
    Therefore, the distribution $\{p_i\}_{i \in [n]}$ computed in line~5 satisfies,
    \[ p_\ell = \frac{\beta_\ell + \gamma \cdot \norm{v_\ell}_2^2}{\sum_{j \in [n]} \beta_j + \gamma \cdot \norm{\V}_F^2} \ge \frac{1}{4} \cdot \frac{\norm{x_\ell}_2^2 + \tilde{\gamma} \cdot \norm{v_\ell}_2^2}{\norm{\X}_F^2 + \tilde{\gamma} \cdot \norm{\V}_F^2}. \]
    
    Furthermore, note that $\sr(\wt{\D}^{-1} \A) \le 2\cdot \sr({\D}^{-1} \A)$.
    Therefore, we can invoke the AMM result from \cref{lem:amm-operator-norm} with matrices $\X^\top = \wt{\D}^{-1} \A$ and $\Y = \V$ and use the precondition of \cref{thm-correctness-outerloop} about the number of samples $m = \Omega\left( \varepsilon^{-2} \log n \cdot (\sr({\D}^{-1} \A) + \sr(\V)) \right) = \Omega\left( \varepsilon^{-2} \log n \cdot (\sr(\wt{\D}^{-1} \A) + \sr(\V)) \right)$ to conclude that the sampling matrix $\bPi$ computed in lines~6-7 satisfies the following with high probability in $n$:
    \[
    \norm{  \wt{\D}^{-1} \A \bPi^\top \cdot \bPi \V - \wt{\D}^{-1} \A \V }_{\op} \le \frac{\varepsilon}{4} \norm{\wt{\D}^{-1}\A }_{\op} \norm{ \V}_{\op} \le \frac{\varepsilon}{2} \norm{ \D^{-1}\A }_{\op} \norm{ \V}_{\op},
    \]
    where the second inequality above follows from the fact that $\norm{\wt{\D}^{-1}\A }_{\op} \le 2 \cdot \norm{ \D^{-1}\A }_{\op}$. The above inequality shows that \cref{eq:error-bound-sampler-attn} holds with high probability in $n$.
    Thus the theorem follows from combining \cref{eq:error-D-approx} and \cref{eq:error-bound-sampler-attn} using triangle inequality.
\end{proof}

\subsection{Proof of \cref{thm-main-attenstion-full-alg}}\label{appendix_proof_main_thrm}

\mainthrmspectralguarantee*

\begin{proof}
    It suffices to run \cref{alg-outer-loop} with some $m = O\left( \varepsilon^{-2} \log n (\sr({\D}^{-1} \A) + \sr(\V)) \right)$ samples and invoke \cref{alg-w-exp-kde} for the calls to \textsc{WExpKDE} made in lines~3-4.
    By \cref{thm-correctness-outerloop} and \cref{thm-corrctness-runtime-wexpkde-alg} along with union bound, the outputs $\bPi$ and $\wt{\D}$ of this procedure satisfy the desired condition of \cref{eq:def-spectral-erro-attention} with probability $\ge 1 - \frac{1}{\poly(n)}$.
    
    \paragraph{Runtime Analysis.} By \cref{thm-corrctness-runtime-wexpkde-alg}, the time to compute $\wt{\D}$ through invoking \textsc{WExpKDE} (i.e., \cref{alg-w-exp-kde}) in line~3 of \cref{alg-outer-loop} is $O \left( n d \cdot \Ccal_{ \frac{\K}{d^{1/4}} , \frac{\Q}{d^{1/4}}, \mathbf{1}_n, \varepsilon, \tau} \right)$.
    Furthermore, time to run \textsc{WExpKDE} in line~4 is $O \left( n d \cdot \Ccal_{\frac{\sqrt{2} \cdot \Q}{d^{1/4}}, \frac{\sqrt{2} \cdot \K}{d^{1/4}}, u, 1, \tau} \right) $, where $u$ is the vector computed in lines~3-4 of \cref{alg-outer-loop}.
    On the other hand, by \cref{thm-corrctness-runtime-wexpkde-alg}, vector $u$ satisfies $\frac{1}{2} v_j \le u_j \le \frac{3}{2} v_j$ for all $j \in [n]$ with probability at least $1 - \frac{1}{\poly(n)}$, where $v$ is the vector defined in the theorem statement. 
    Thus, using the definition of $\Ccal_{\frac{\sqrt{2} \cdot \Q}{d^{1/4}}, \frac{\sqrt{2} \cdot \K}{d^{1/4}}, u, 1, \tau}$ in \cref{eq:runtime-bound-WExpKDE} we can show that the aforementioned runtime is bounded by $O \left( n d \cdot \Ccal_{\frac{\sqrt{2} \cdot \Q}{d^{1/4}}, \frac{\sqrt{2} \cdot \K}{d^{1/4}}, v, 1, \tau} \right) $.
    
    Finally, the time to generate $m$ samples in line~6 of \cref{alg-outer-loop} is $O(m + n)$, using the sampling method developed by~\citet{hagerup1993maintaining}. The total runtime is obtained by summing up these terms.
\end{proof}

\subsection{Proof of \cref{corr-simplified-runtime}}\label{appndx-proof-corr}
\corrsimplifiedruntimr*
\begin{proof}
    First recall that the diameter of the datasets $\Q, \K$ is $\max_{i,j \in [n]} \norm{k_i - q_j}_2^2 = \gamma \sqrt{d} \log n$ for some $\gamma>0$.
    For any $i , j \in [n]$, using the fact that $\norm{k_i - q_j}_2^2 \le \gamma \sqrt{d} \log n$, we have,
    \begin{align*}
        \exp\left( \frac{1}{\sqrt{d}}\langle k_j, q_i \rangle \right) &= \exp\left( \frac{-1}{2\sqrt{d}} \norm{k_j - q_i}_2^2 \right) \cdot \exp\left( \frac{\norm{k_j}^2 + \norm{q_i}^2}{2\sqrt{d}} \right) \\
        &\ge n^{-\gamma/2} \cdot \exp\left( \frac{\norm{k_j}^2 + \norm{q_i}^2}{2\sqrt{d}} \right).
    \end{align*}
    Therefore, summing the above inequality over all $j \in [n]$ gives,
    \[ 
    \sum_{j \in [n]} \exp\left( \frac{1}{\sqrt{d}}\langle k_j, q_i \rangle \right) \ge n^{-\gamma/2} \cdot \sum_{j \in [n]}  \exp\left( \frac{\norm{k_j}^2 + \norm{q_i}^2}{2\sqrt{d}} \right).
    \]
    The above inequality holds for every $i \in [n]$. This inequality implies that the following set is empty for any $\mu \le n^{-1 - \gamma/2}$,
    \[
    \left\{ i\in[n]:  \frac{\sum_{j \in [n]} \exp\left( \frac{1}{\sqrt{d}}\langle k_j, q_i \rangle \right)}{ \sum_{j \in [n]} \exp\left( \frac{\norm{k_j}^2 + \norm{q_i}^2}{2\sqrt{d}} \right)} < n \cdot \mu \right\} = \emptyset.
    \]
    Thus, $\Ccal_{ \frac{\K}{d^{1/4}} , \frac{\Q}{d^{1/4}}, \mathbf{1}_n, \varepsilon, \tau}$ defined as per \cref{eq:runtime-bound-WExpKDE} is bounded as follows,
    \begin{align*}
        \Ccal_{ \frac{\K}{d^{1/4}} , \frac{\Q}{d^{1/4}}, \mathbf{1}_n, \varepsilon, \tau} &= \min_{\mu > 0} \, \varepsilon^{-2}  \mu^{-\tau} + \left| \left\{ i\in[n]:  \frac{\sum_{j \in [n]} \exp\left( \frac{1}{\sqrt{d}}\langle k_j, q_i \rangle \right)}{ \sum_{j \in [n]} \exp\left( \frac{\norm{k_j}^2 + \norm{q_i}^2}{2\sqrt{d}} \right)} < n \mu \right\} \right| \\
        &\le \varepsilon^{-2} \cdot  n^{\tau(1 + \gamma/2)}.
    \end{align*}
    
    Similarly, because $v_{j} > 0$ for every $j\in[n]$, we can show that, for any $i \in [n]$,
    \[
    \sum_{j \in [n]} v_{j} \exp\left( \frac{2}{\sqrt{d}}\langle q_j, k_i \rangle \right) \ge n^{-\gamma} \cdot \sum_{j \in [n]} v_{j} \exp\left( \frac{\norm{q_j}^2 + \norm{k_i}^2}{\sqrt{d}} \right).
    \]
    As a result, the following set is empty for any $\mu \le n^{-1-\gamma}$,
    \[
    \left\{ i\in[n]:  \frac{\sum_{j \in [n]} v_{j} \cdot \exp\left( \frac{2}{\sqrt{d}} \langle q_j, k_i \rangle \right)}{ \sum_{j \in [n]} v_{j} \exp\left( \frac{\norm{q_j}^2 + \norm{k_i}^2}{\sqrt{d}} \right)} < n \cdot \mu \right\} = \emptyset.
    \]
    So, $\Ccal_{\frac{\sqrt{2} \cdot \Q}{d^{1/4}}, \frac{\sqrt{2} \cdot \K}{d^{1/4}}, v, 1, \tau}$ defined as per \cref{eq:runtime-bound-WExpKDE} is bounded as follows,
    \begin{align*}
        \Ccal_{\frac{\sqrt{2} \cdot \Q}{d^{1/4}}, \frac{\sqrt{2} \cdot \K}{d^{1/4}}, v, 1, \tau} &= \min_{\mu > 0} \,  \mu^{-\tau} + \left| \left\{ i\in[n]:  \frac{\sum_{j \in [n]} v_{j} \cdot \exp\left( \frac{2}{\sqrt{d}} \langle q_j, k_i \rangle \right)}{ \sum_{j \in [n]} v_{j} \exp\left( \frac{\norm{q_j}^2 + \norm{k_i}^2}{\sqrt{d}} \right)} < n \cdot \mu \right\} \right|\\
        &\le n^{\tau(1+\gamma)}.
    \end{align*}
    
    Therefore, the total runtime of \cref{thm-main-attenstion-full-alg} is bounded by 
    \[
    O\left( m + nd \cdot \left( \Ccal_{ \frac{\K}{d^{1/4}} , \frac{\Q}{d^{1/4}}, \mathbf{1}_n, \varepsilon, \tau} + \Ccal_{\frac{\sqrt{2} \cdot \Q}{d^{1/4}}, \frac{\sqrt{2} \cdot \K}{d^{1/4}}, v, 1, \tau} \right) \right) = O\left(m + nd \cdot \left( n^{\tau(1+\gamma)} + n^{\tau(1+\gamma/2)}/\varepsilon^2 \right) \right),
    \]
    which completes the proof.
\end{proof}

\section{Additional Results on BigGAN Image Generations}\label{sec:appendix_biggan}

Images in \cref{fig:biggan} are randomly subset from $2,000$ generations from BigGAN~\cite{yuan2021tokens}\footnote{\url{https://github.com/huggingface/pytorch-pretrained-BigGAN}} with the exact attention computation and its various approximations including KDEformer (our), Performer~\cite{choromanski2020rethinking}, Reformer~\cite{kitaev2019reformer} and ScatterBrain~\cite{chen2021scatterbrain}. One can observe that our KDEformer generates more natural and realistic images than other methods by a large margin, and in many cases it is even better than the exact computation. This means that it has much less running time and memory, but it has produced a higher quality and more realistic image in the end. Also, note that the hyperparameters of our approach were not fine-tuned.
	
\begin{figure}[t]
    \centering
    \includegraphics[width=0.9\textwidth]{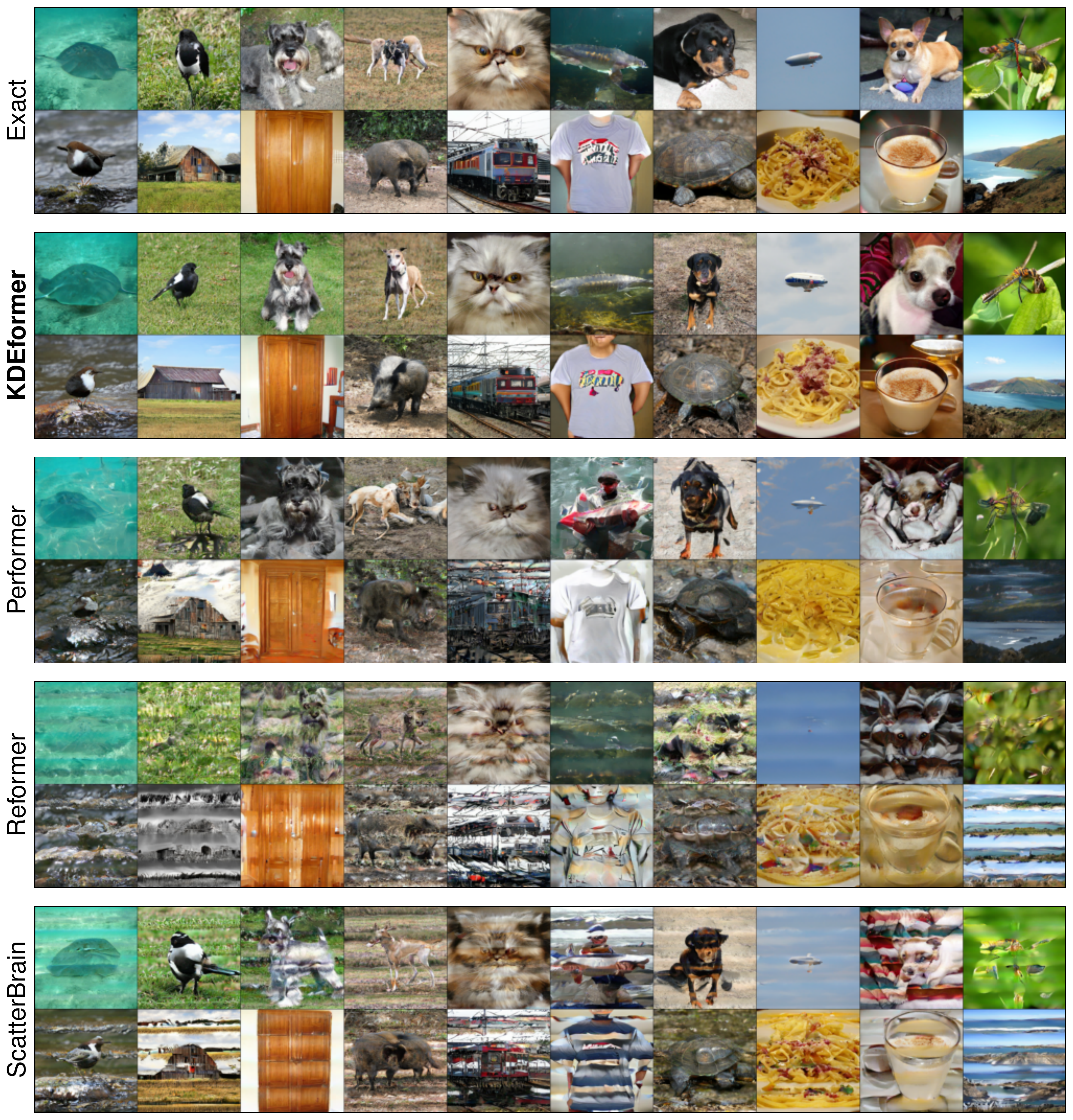}
    \vspace{-0.15in}
    \caption{Images generations from the pre-trained BigGAN with the exact attention (top) and drop-in replacement with its approximations including our KDEformer (second row), Performer (third row), Reformer (fourth row) and ScatterBrain (bottom).}
    \label{fig:biggan}
\end{figure}
	
\end{document}